\documentclass[english,10pt]{article}
\usepackage[papersize={19cm,25cm},body={15cm,21cm},centering]{geometry}
\usepackage[T1]{fontenc}
\usepackage[latin9]{inputenc}
\usepackage{amsmath}
\usepackage{amssymb}
\usepackage{babel}
 \usepackage{graphicx}
\usepackage{epsfig,subfigure,color}
\usepackage{graphics}
 \usepackage{graphicx}
\usepackage{amsmath,amsfonts,amsthm,latexsym}
\usepackage{mathrsfs,amsbsy}
\usepackage{multirow}
\usepackage{algorithm}
\usepackage{algorithmic}
\newtheorem{theorem}{Theorem}
\newtheorem{lemma}{Lemma}
\newtheorem{Definition}{Definition}
\newtheorem{property}{Property}
\usepackage{cite}

\begin{document}
%
\title{Orthogonal Nonnegative Tucker Decomposition}
%
%
\date{}
\author{Junjun Pan,
        Michael K. Ng,
        Ye Liu,
        Xiongjun Zhang,
         Hong Yan
       \thanks{J. Pan is with Department of Mathematics and Operational Research
Facult\'e polytechnique, Universit\'e de Mons, 7000, Belgium. (e-mail: Junjun.PAN@umons.ac.be).}
\thanks{ M. K. Ng, and Y. Liu are with the Department of Mathematics, The University of Hong Kong, Pokfulam, Hong Kong.
(e-mail:   mng@math.hku.hk; 16482549@life.hkbu.edu.hk).}
\thanks{X. Zhang is with the School of Mathematics and Statistics and Hubei Key Laboratory of Mathematical Sciences,
Central China Normal University, Wuhan 430079, China. (e-mail: xjzhang@mail.ccnu.edu.cn).}
\thanks{H. Yan is with the Department of Electronic Engineering, City University of Hong Kong, Kowloon, Hong Kong (e-mail:h.yan@cityu.edu.hk).}}

\maketitle

\begin{abstract}
In this paper, we study the nonnegative tensor data and propose an orthogonal nonnegative Tucker decomposition (ONTD).   We discuss some properties of ONTD and develop a convex relaxation algorithm of the augmented Lagrangian function to solve the optimization problem.  The convergence of the algorithm is given. We employ ONTD on the image data sets from the real world applications including face recognition, image representation, hyperspectral unmixing. Numerical results are shown to illustrate the effectiveness of the proposed algorithm.
\end{abstract}


%

\section{Introduction}

Given a nonnegative matrix $\mathbf{A}\in\mathbb{R}^{m\times n}_+$ and integer $r$, nonnegative matrix factorization (NMF) is the problem of searching for basic matrix $\mathbf{U}\in \mathbb{R}^{m\times r}_+$ and coefficient matrix $\mathbf{V}\in \mathbb{R}^{r\times n}_+$ such that $\mathbf{A}\approx\mathbf{ UV}$. In many data analysis problems, the columns of $\mathbf{A}$ are corresponding to data points, for instance, images of pixel intensities. NMF has been successfully applied into many fields including image processing, text data mining and so on. It has been demonstrated that NMF is a powerful technique for dimension reduction. Compared to other well-known method, like singular value decomposition or principal component analysis, NMF is able to give more interpretable results due to its combinations of nonnegative basic vectors.

In general, NMF is NP-hard and the solution is not unique. It is necessary to impose additional constraints on the factor matrix like orthogonality constraints. Precisely, given $\mathbf{A}\in\mathbb{R}^{m\times n}_+$, solve
$$
\min_{\mathbf{U}\in \mathbb{R}^{m\times r}_+,\mathbf{V}\in \mathbb{R}^{r\times n}_+}\|\mathbf{A}-\mathbf{U}\mathbf{V}\|^2_F,\quad \mbox{s.t.}\quad \mathbf{U}^T\mathbf{U}=\mathbf{I}.
$$
We call the above problem orthogonal nonnegative matrix factorization(ONMF). The orthogonal constraints guarantee the uniqueness of the solution. There are many methods \cite{Choi08,Pomp14,Ding06} and most are the multiplicative update algorithms derived from NMF. Until recently, Pan and Ng \cite{Pan18} investigated the properties of ONMF and present a new method called SN-ONMF for finding the factorization. They used the sparsity and nuclear norm optimization to solve ONMF problem.

The orthogonality constraints make sense in many practical applications. In \cite{Ding06}, the equivalence of ONMF problem and K-means clustering has been well discussed. In document classification, each entry $A(i,j)$ indicates the importance of word $i$ in document $j$. Each row of data matrix stands for a document, each column stands for a word.  $\mathbf{U}$ is the document cluster indicator matrix in ONMF model, in other words, ONMF aims to find a document clustering space $\mathbf{U}$, and the coefficient matrix $\mathbf {V}$ can be obtained by projecting the data onto $\mathbf{U}$. In addition to understanding the cluster of documents, one may also require the cluster  of words. Considering this problem, Ding et al proposed the following nonnegative tri-factor decomposition in \cite{Ding06},
\begin{equation}\label{trionmf}
\min_{\mathbf{U}\in \mathbb{R}^{m\times r_1}_+ ,\mathbf{S}\in \mathbb{R}^{r_1\times r_2}, \mathbf{V}\in \mathbb{R}^{r_2\times n}_+}\|\mathbf{A}-\mathbf{U}\mathbf{S}\mathbf{V}\|^2_F,\quad \mbox{s.t.}\quad  \mathbf{U}^T\mathbf{U}=\mathbf{I},\quad \mathbf{V}^T\mathbf{V}=\mathbf{I}.
\end{equation}
$\mathbf{U}$ provides row clusters and $\mathbf{V}$ provides column clusters. For coefficient matrix $\mathbf{S}$, each entry $s_{i,j}$ can be regarded as the connection weight between column cluster $i$ and row cluster $j$.

Nowadays, data that comes from many fields are more naturally represented as multidimensional data which refers to tensor, for example, video data, hyperspectral data, fMRI data and so on. In this paper, we generalize  model (\ref{trionmf}) to tensor data, i.e., given a nonnegative tensor $\mathcal{A}\in \mathbb{R}^{I_1\times I_2\times \cdots\times I_d}_+$ and the integer rank $(J_1,J_2,\ldots, J_d)$, solve
\begin{equation}\label{e1}
\begin{split}
\min& \ \|\mathcal{A}-\mathcal{S}\times_1 \mathbf{U}^{(1)}\times_2 \mathbf{U}^{(2)}\cdots \times_d \mathbf{U}^{(d)}\|^2_F\\
\mbox{s.t.}&\ \mathcal{S}\in \mathbb{R}^{J_1\times J_2\times\cdots\times J_d}_+, ~\mathbf{U}^{(n)}\in \mathbb{R}^{I_n\times J_n}_+,~\mathbf{U}^{(n)T}\mathbf{U}^{(n)}=\mathbf{I},~n=1,2,\ldots,d.
\end{split}
\end{equation}
where $\times_n$ denotes the mode-$n$ matrix product of a tensor defined by
$$
(\mathcal{S}\times_n \mathbf{U}^{(n)})_{j_1\cdots j_{n-1}i_{n}j_{n+1}\cdots j_d}=\sum^{J_n}_{j_n=1}s_{j_1\cdots j_{n-1}j_nj_{n+1}\cdots j_d}u^{(n)}_{i_n,j_n}.
$$
For simplicity, we call the above model orthogonal nonnegative Tucker decomposition (ONTD) model. $\mathbf{U}^{(n)}$ gives the clusters of the $n$-th dimension. Each entry $s_{j_1,j_2,\cdots,j_d}$ represents the joint connection weight of the corresponding cluster along dimensions from $1$ to $d$. If some factor matrices are equal to identity matrix $\mathbf{I}$, we call the model partial ONTD model.

The ONTD model makes sense in several applications. In image classification, for some image sequences containing different illuminations, motions and subjects, ONTD model gives illumination clusters, motion clusters and subject clusters. We can also know the connection weight between them. For video data sets which contains different types of human actions, scenarios and different subjects, ONTD model helps us to know the clusters of actions, scenarios and subjects. If one only consider the clusters of actions and subjects, he can use partial ONTD model, i.e., set the scenarios factor matrix $\mathbf{U}$ to be identity matrix.

ONTD model not only helps to keep the inherent tensor structure but also well performs in data compression. As we all know that tensor data need huge storages due to high dimensionality. While in ONTD, only $(J_1\cdots J_d+\sum^d_{n=1}I_nJ_n)$ memories are required for the tensor of $(I_1,\cdots, I_d)$. It would save a lot of memory compared with the original storage $(I_1I_2\cdots I_d)$, especially when $(J_1,\cdots,J_d)$ is small.

ONTD model is related to multilinear singular value decomposition (HOSVD) which enforces the factor matrices in Tucker decomposition into orthogonal matrix and is well discussed in \cite{DeLa00best,DeLa00}. Higher-order orthogonal iteration (HOOI) in \cite{DeLa00best} is proposed to find the best rank $(J_1,J_2,\cdots,J_d)$ approximation of tensor $\mathcal{A}$. The difference between HOSVD and our model is the nonnegativity of factor matrices are imposed in ONTD model, which lead to easily interpret. In \cite{kim2007nonnegative}, Kim and Choi developed nonnegative Tucker decomposition (NTD) while they did not consider the orthogonality on the factor matrices. ONTD model takes advantages from both orthogonality and nonnegativity constraints. We can get the clustering information from the factor matrices $\{U^{(i)}\}^d_{i=1} $ and their joint connection weight from the core tensor $\mathcal{S}$ at the same time.

To solve the model (\ref{e1}), we first discuss the properties of nonnegative orthogonal Tucker format tensor, i.e., $\mathcal{A}=\mathcal{S}\times_1 \mathbf{U}^{(1)}\times_2 \mathbf{U}^{(2)}\cdots \times_d \mathbf{U}^{(d)}$ where $\{\mathbf{U}^{(n)}\}^d_{n=1}$ are orthogonal nonnegative matrices. We utilize the properties of nonnegative orthonormal factor matrix and present a structured convex optimization algorithm. The convergence of the algorithm will be discussed and shown. The proposed ONTD method is then applied to face recognition, image representation and hyperspectral unmixing problems. Numerical results demonstrate a good performance of our model. We summarize the main contributions of this paper as follows.

1) We propose an orthogonal nonnegative Tucker decomposition model (ONTD) which projects the tensor objects into the tensor with small size for dimension reduction and preserves the structure information as well.

2) We present some properties of ONTD model, develop a structured convex optimization algorithm and analysis the convergence.

3) Numerical examples  from the real world applications have been conducted to demonstrate the effectiveness of the proposed method. The results show that the ONTD outperforms existing methods such as PCA, NMF and NTD.

The rest of the paper is organized as follows. In Section 2, we present several  properties of orthogonal nonnegative tensors. In Section 3  we propose an optimization model and present an algorithm. Meanwhile, the convergence is discussed. In Section 4, we apply the algorithm on real data sets from face recognition, image representation and hyperspectral unmixing problem. Their numerical results are shown. The concluding remarks are given in Section 5.
\section{Properties of Nonnegative Orthogonal Tucker Tensor}
Given a nonnegative orthogonal Tucker tensor $\mathcal{A}\in \mathbb{R}^{I_1\times \cdots\times I_d}$,
\begin{equation}\label{NOTT}
\mathcal{A}=\mathcal{S}\times_1 \mathbf{U}^{(1)}\times_2 \mathbf{U}^{(2)}\cdots \times_d \mathbf{U}^{(d)}
\end{equation}
$\{\mathbf{U}^{(n)}\in \mathbb{R}^{I_n\times J_n}\}^d_{n=1}$ are orthogonal nonnegative matrices, $(J_1,J_2,\cdots,J_n)$ refers to as multilinear rank. We first introduce the following two properties of  orthogonal nonnegative matrix that from \cite{Pan18}.

\begin{lemma}\label{leOrg}
For $n=1,2,\ldots,d$, each row of ${\bf U}^{(n)}$ has at most one nonzero element.
\end{lemma}
\begin{lemma}\label{crin}
For $n=1,2,\ldots,d$, ${\bf K}^{(n)} = {\bf U}^{(n)} {\bf U}^{(n)T}$ satisfies
$$
{\bf K}^{(n)T}={\bf K}^{(n)} ~{\rm and} ~({\bf K}^{(n)})^2={\bf K}^{(n)}, ~ {\rm with} ~~ 0\leq k^{(n)}_{i,j}\leq 1.
$$
In addition, the trace of ${\bf K}^{(n)}$ is equal to $J_n$, i.e., $tr({\bf K}^{(n)})=J_n$ and $\| {\bf K}^{(n)} \|_1 \le I_n$, where $I_n$ and $J_n$ are the dimensions of ${\bf U}^{(n)}$. Moreover, 1 is the $J_n$ repeated eigenvalues of ${\bf K}^{(n)}$ and  the columns of ${\bf U}^{(n)}$ are the corresponding eigenvectors.
\end{lemma}

From Lemma 1, we know $\mathbf{U}^{(n)}$ has the following structure:
\begin{equation}\label{Un}
\mathbf{U}^{(n)}=\Pi_r\left(
                        \begin{array}{cccc}
                          \mathbf{u}^{(n)}_1 & 0 & \cdots & 0 \\
                          0 & \mathbf{u}^{(n)}_2 & \cdots  & 0 \\
                          \vdots& \vdots &  \ddots & \vdots \\
                          0 & 0 & \cdots & \mathbf{u}^{(n)}_{J_n}\\
                        \end{array}
                      \right)\Pi_c, \quad\mathbf{u}^{(n)}_j=\left(
                                                                 \begin{array}{c}
                                                                   u^{(n)}_{1,j} \\
                                                                   u^{(n)}_{2,j} \\
                                                                   \vdots \\
                                                                   u^{(n)}_{l_j,j} \\
                                                                 \end{array}
                                                               \right),
\end{equation}
where $\quad j=1,\cdots,J_n,\quad n=1,2,\cdots, d$, and $\sum_jl_j=I_n$.  $\Pi_c$ and $\Pi_r$ are permutation matrix.  $\mathbf{U}^{(n)}$ can be regarded as class indicator matrix along $n$-th direction, $\mathbf{u}^{(n)}_j$ represents $j$-th class containing $l_j$ rows.

For tensor $\mathcal{A}$, its $n$-th unfolding $\mathbf{A}_{(n)}\in \mathbb{R}^{I_n\times (I_{n+1}\cdots I_{N}I_1\cdots I_{n-1})}$, follow (\ref{NOTT}), we can say that the $I_n$ rows of $\mathbf{A}_{(n)}$ can be classified into  $J_n$ classes. $\mathbf{U}^{(n)}$ is its corresponding class indicator matrix. For simplicity, denote $\mathbf{A}_{(n)}(t,:)$ as the $t$-th row of $\mathbf{A}_{(n)}$, and $\mathbf{A}_{(n)}(T_j,:)$ be the row set  belongs to $j$-th class, with the row index set $T_j$ of cardinality $l_j$,  $j\in \{1,\cdots J_n\}$.

Authors in \cite{Pan18} proved that the rows and columns of the same groups are proportional if matrix $\mathbf{A}$ is orthogonal decomposable, i.e., $\mathbf{A}=\mathbf{B}\mathbf{C}$ where $\mathbf{B}$ is orthogonal nonnegative matrix. Utilize this result, we give the similar property of $\mathcal{A}$ as follows.
\begin{property}
 Given a nonnegative orthogonal Tucker tensor $\mathcal{A}$, for $n$-th unfolding $\mathbf{A}_{(n)}$, $n=1,2,\ldots,d$, the rows belong to the same class are proportional.
\end{property}
\begin{proof}
From Lemma 1, we know that $\mathbf{U}^{(n)}\in \mathbb{R}^{I_n\times J_n}$ has at most one nonzero element in each row, i.e., in the $i$-th row,
$
u^{(n)}_{i,j}=
\left\{\begin{array}{cll}
u^{(n)}_{i,j^*}, & \mbox{if} \ j=j^*,\\
0, & \mbox{otherwise}.\\
\end{array}\right.
$  From Tucker format, we get that
\begin{equation}\label{aeq}
a_{i_1,i_2,\cdots, i_d}=\sum_{j_1,j_2,\cdots,j_d}s_{j_1,j_2,\cdots,j_d}u^{(1)}_{i_1,j_1}\cdots,u^{(d)}_{i_d,j_d}=s_{j^*_1,\cdots,j^*_d}u^{(1)}_{i_1,j^*_1}\cdots,u^{(d)}_{i_d,j^*_d}
\end{equation}
Fix all the index but $i_n$,
\begin{equation*}
\frac{a_{i_1,\cdots,i_{n-1},i^1_n,i_{n+1}\cdots i_d}}{a_{i_1,\cdots,i_{n-1},i^2_n,i_{n+1} \cdots i_d}}=\frac{s_{j^*_1,\cdots, j^{*1}_n\cdots j^*_d}u^{(n)}_{i^1_nj^{*1}_n}}{s_{j^*_1,\cdots, j^{*2}_n\cdots j^*_d}u^{(n)}_{i^2_nj^{*2}_n}},
\end{equation*}
if  $j^{*1}_n=j^{*2}_n$, then,
\begin{equation}\label{An}
\frac{a_{i_1,\cdots,i_{n-1},i^1_n,i_{n+1}\cdots i_d}}{a_{i_1,\cdots,i_{n-1},i^2_n,\cdots i_{n+1} i_d}}=\frac{u^{(n)}_{i^1_nj^{*1}_n}}{u^{(n)}_{i^2_nj^{*2}_n}},
\end{equation}
(\ref{An}) implies that $i^1_n$-th and $i^2_n$-th row of $\mathbf{A}_{(n)}$ are in the same class, and proportional. The result follows.
\end{proof}

\begin{property}
If $\mathcal{A}$ is a nonnegative orthogonal Tucker tensor, then its factor matrices $\{\mathbf{U}^{(n)}\}^d_{n=1}$ can be given explicitly by $\{\mathbf{A}_{(n)}\}^d_{n=1}$.
\end{property}
\begin{proof}
 For $U^{(n)}$ defined in (\ref{Un}), the $j$-th column $\mathbf{u}^{(n)}_j=(
                                                                 \begin{array}{cccc}
                                                                   u^{(n)}_{1,j} &  u^{(n)}_{2,j} &\cdots & u^{(n)}_{l_j,j}  \\                                                         \end{array}
                                                               )^T$ can be constructed in the following way. Without loss of generality, let
                                                               $u^{(n)}_{1,j}\neq 0$, from (\ref{An}),
$$
\frac{u^{(n)}_{t,j}}{u^{(n)}_{1,j}}=\frac{a_{i_1,\cdots, i_{n-1},t,i_{n+1},\cdots, i_d}}{a_{i_1,\cdots, i_{n-1},1,i_{n+1},\cdots, i_d}}\doteq \alpha_{t,j},
$$
where $t\in \{1,2,\cdots, l_j\}$, $j\in \{1,2,\cdots, J_n\}$. Because of $\mathbf{U}^{(n)T}\mathbf{U}^{(n)}=I$,  $\mathbf{u}^{(n)}_j$ can be constructed by letting
$$
u^{(n)}_{t,j}=\frac{\alpha_{t,j}}{\alpha_j},\quad ~\alpha_j=\sqrt{\sum^{l_j}_{t=1}\alpha^2_{t,j}}.
$$
\end{proof}

\begin{property}
  If $\mathcal{A}$ is a nonnegative orthogonal Tucker tensor, then for core tensor $\mathcal{S}$, the Frobenius norm of $j$-th row of $\mathbf{S}_{(n)}$ equals to that of $\mathbf{A}_{(n)}(T_j,:)$. More over, the Frobenius norm of $\mathcal{S}$ equals to that of $\mathcal{A}$.
\end{property}
\begin{proof}
For $\mathbf{A}_{(n)}$ whose rows are classified into $J_n$  classes, let $\{a_{i_1,\cdots, i_{n-1}, t,i_{n+1},\cdots,i_{d}}\}^{l_j}_{t=1}$ belong to the $j$-th class, from (\ref{aeq}),
  \begin{eqnarray}
  \notag
    \|\mathbf{A}_{(n)}(T_j,:)\|^2_F=&\sum_{i_1,\cdots,i_{n-1},t,i_{n+1},\cdots,i_d} a^2_{i_1,\cdots, i_{n-1}, t,i_{n+1},\cdots,i_{d}} \\
     \notag =&\sum_{i_1,\cdots,i_{n-1},t,i_{n+1},\cdots,i_d}(s_{j^*_1,\cdots,j^*_d}u_{i_1,j^*_1}\cdots u_{i_d,j^*_d})^2\\
      \notag
  =&\sum_{i_1,\cdots,i_{n-1},i_{n+1},\cdots,i_d}\sum_{t}(s^2_{j^*_1,\cdots,j^*_d}u^2_{t,j})u^2_{i_1,j^*_1}\cdots u^2_{i_{n-1},j^*_{n-1}}u^2_{i_{n+1},j^*_{n+1}}\cdots u^2_{i_d,j^*_d}\\
  =&\sum_{i_1,\cdots,i_{n-1},i_{n+1},\cdots,i_d} s^2_{j^*_1,\cdots,j^*_{n-1},j,j^*_{n+1},\cdots,j^*_d}u^2_{i_1,j^*_1}\cdots u^2_{i_{n-1},j^*_{n-1}}u^2_{i_{n+1},j^*_{n+1}}\cdots u^2_{i_d,j^*_d}\label{pro3_1}\\
   \notag
  =&\sum_{i_2,\cdots,i_{n-1},i_{n+1}\cdots i_d}(\sum_{i_1}s^2_{j^*_1,\cdots,j^*_{n-1},j,j^*_{n+1},\cdots,j^*_d} u^2_{i_1,j^*_1})\cdots u^2_{i_{n-1},j^*_{n-1}}u^2_{i_{n+1},j^*_{n+1}}\cdots u^2_{i_d,j^*_d}\\
  =&\sum_{i_2,\cdots,i_{n-1},i_{n+1}\cdots i_d}(\sum_{j^*_1}s^2_{j^*_1,\cdots,j^*_{n-1},j,j^*_{n+1},\cdots,j^*_d})\cdots u^2_{i_{n-1},j^*_{n-1}}u^2_{i_{n+1},j^*_{n+1}}\cdots u^2_{i_d,j^*_d}\label{pro3_2}\\
   \notag
  =&\cdots
  =\sum_{j^*_1,\cdots,j^*_{n-1}j^*_{n+1}\cdots j^*_{d}}s^2_{j^*_1,\cdots,j^*_{n-1},j,j^*_{n+1},\cdots,j^*_d}=\|\mathbf{S}_{(n)}(j,:)\|^2_F.
\end{eqnarray}
Because of $\mathbf{U}^{(n)T}\mathbf{U}^{(n)}=\mathbf{I}$, $\sum_tu^2_{t,j}=1$, thus the equality (\ref{pro3_1}) is established, moreover, since $i_1$ values from 1 to $I_1$, the corresponding $j^*$ hence goes through from 1 to $J_1$, equality (\ref{pro3_2}) holds.

 The Frobenius norm of $\mathcal{S}$ equals to that of $\mathcal{A}$ follows the summation from $j=1$ to $j=J_n$ of both sides of the above equality.
\end{proof}

\section{The Optimization Method}

In this section, we develop the optimization method for solving (\ref{e1}).
Equation (\ref{e1}) can be rewritten as follows:
\begin{equation}\label{e2}
\begin{split}
  \min& \ \|\mathbf{A}_{(n)}-\mathbf{U}^{(n)}\mathbf{S}_{(n)}(\mathbf{U}^{(n+1)}\otimes \mathbf{U}^{(n+2)}\cdots \otimes\mathbf{U}^{(d)}\otimes\mathbf{U}^{(1)}\otimes\mathbf{U}^{(2)} \cdots \otimes\mathbf{U}^{(n-1)})^{T}\|^2_F\\
   \mbox{s.t.} & \ \mathbf{U}^{(n)}\in \mathbb{R}^{I_n\times J_n}_+,~~\mathbf{U}^{(n)T}\mathbf{U}^{(n)}=\mathbf{I},\quad \mathbf{S}_{(n)}\in \mathbb{R}^{J_n\times J_{n+1}\cdots J_d J_1\cdots J_{n-1}}_+.
\end{split}
\end{equation}
where $\otimes$ denotes the Kronecker product. We remark that
(\ref{e2}) is valid for $1 \le n \le d$.
For simplicity,
we let
$$
\mathbf{W}^{(n)}=\mathbf{S}_{(n)}(\mathbf{U}^{(n+1)}\otimes \cdots
  \cdots \otimes\mathbf{U}^{(d)}\otimes\mathbf{U}^{(1)}\otimes \cdots \otimes\mathbf{U}^{(n-1)})^{T}.
  $$
Therefore, each factor matrix $\mathbf{U}^{(n)}$ can be obtained by solving the following subproblem:
\begin{equation}\label{e3}
\begin{split}
&\min\|\mathbf{A}_{(n)}-\mathbf{U}^{(n)}\mathbf{W}^{(n)}\|^2_F \\
&\mbox{s.t.}~  \mathbf{U}^{(n)}\in \mathbb{R}^{I_n\times J_n}_+,~\mathbf{U}^{(n)T}\mathbf{U}^{(n)}=\mathbf{I},
 ~\mathbf{W}^{(n)}\in \mathbb{R}^{J_n\times I_{n+1}\cdots I_{d}I_1\cdots I_{n-1}}_+.
 \end{split}
\end{equation}
We note that (\ref{e3}) is an orthogonal nonnegative matrix factorization problem.
Similar to \cite{Pan18}, we propose to solve the following optimization problem
instead\footnote{
Assume that $\mathcal{A}$ is orthogonally decomposable, i.e.,
${\cal A} = {\cal S} \times_1 {\bf U}^{(1)} \cdots \times_d {\bf U}^{(d)}$
with $\mathbf{U}^{(n)T}\mathbf{U}^{(n)}=\mathbf{I}$ for $1 \le n \le d$. It is clear that
$\mathbf{A}_{(n)}=\mathbf{U}^{(n)}\mathbf{W}^{(n)}$ for $1 \le n \le d$.
Since ${\bf U}^{(n)}$ is orthogonal, we know that
$\mathbf{W}^{(n)}=\mathbf{U}^{(n)T}\mathbf{A}_{(n)}$. It implies that
${\bf A}_{(n)} = {\bf U}^{(n)T} {\bf U}^{(n)} {\bf A}_{(n)}$. Here we propose to minimize the difference
between ${\bf A}_{(n)}$ and ${\bf U}^{(n)T} {\bf U}^{(n)} {\bf A}_{(n)}$.}
\begin{equation}\label{e4}
   \mathbf{U}^{(n)}=\mathop{\arg\min}\big\{ \|\mathbf{A}_{(n)}-\mathbf{U}^{(n)}\mathbf{U}^{(n)T}\mathbf{A}_{(n)}\|^2_F:
  \mathbf{U}^{(n)}\in \mathbb{R}^{I_n\times J_n}_+, \mathbf{U}^{(n)T}\mathbf{U}^{(n)}=\mathbf{I}\big\}.
\end{equation}
Next we study how to solve (\ref{e4}) efficiently.

\subsection{The Factor Matrix}

By using Lemma \ref{leOrg} and Lemma \ref{crin}, we rewrite problem (\ref{e4}) as follows:
for $n=1,2,\ldots, d$,
\begin{equation}\label{md1}
\begin{split}
    \min_{\mathbf {K}^{(n)}} &\
    F( \mathbf {K}^{(n)})=\frac{1}{2}\| \mathbf{A}_{(n)}- \mathbf{K}^{(n)}\mathbf{A}_{(n)}\|^2_F\\
    \mbox{s.t.} &\ tr(\mathbf{K}^{(n)})=J_n, \quad {\bf K}^{(n)}={\bf K}^{(n)T},
     \ ({\bf K}^{(n)})^2={\bf K}^{(n)},\quad {\bf K}^{(n)}\geq 0.
\end{split}
\end{equation}
According to Lemma 2, $\mathbf{K}^{(n)}$ has a block-like structure and $\|\mathbf{K}^{(n)}\|_1\leq I_n$. We therefore expect many entries of ${\bf K}^{(n)}$ are zero,
and present the following convex relaxation model ,
\begin{equation}\label{md3}
\begin{split}
    \min_{{\bf K}^{(n)}} &\
    F( {\bf K}^{(n)})=\frac{1}{2}\| \mathbf{A}_{(n)}- \mathbf{K}^{(n)}\mathbf{A}_{(n)}\|^2_F +\theta\| {\bf K}^{(n)} \|_1\\
    \mbox{s.t.} & \ tr({\bf K}^{(n)})=J_n,\quad {\bf K}^{(n)}={\bf K}^{(n)T},   \ \mathbf{0}\preceq {\bf K}^{(n)}\preceq \mathbf{I}, \quad {\bf K}^{(n)}\geq 0.
\end{split}
\end{equation}
where $\|\cdot\|_1$ is $\ell_1$ norm of matrix, $\mathbf{0}\preceq {\bf K}^{(n)}\preceq \mathbf{I}$ denotes that matrix $\mathbf{K}^{(n)}$  and matrix $\mathbf{I}-\mathbf{K}^{(n)}$ are  positive semidefinite. It is the convex hull of $\mathbf{K}^2 = \mathbf{K}$,
which leads to the convex problem (\ref{md3}). The use of $\| {\bf K}^{(n)} \|_1$ is to enforce the sparsity of ${\bf K}^{(n)}$ and $\theta$ is a positive number to control the balance among the two terms in the objective function.

Let $\mathbf{K}^{(n)}=\mathbf{X}^{(n)}$, $\mathbf{K}^{(n)}=\mathbf{Z}^{(n)}$, $\mathbf{K}^{(n)}=\mathbf{M}^{(n)}$, we have
\begin{equation}\label{cmd1}
\begin{split}
    \min_{{\bf K}^{(n)}} & \
    F( {\bf K}^{(n)})=\frac{1}{2}\| \mathbf{A}_{(n)}- \mathbf{K}^{(n)}\mathbf{A}_{(n)}\|^2_F+\theta\| {\bf X}^{(n)} \|_1\\
    \mbox{s.t.} & \ \mathbf{K}^{(n)}-\mathbf{X}^{(n)}=0,  \mathbf{K}^{(n)}-\mathbf{Z}^{(n)}=0, \mathbf{K}^{(n)}-\mathbf{M}^{(n)}=0,\\
     & \ tr({\bf K}^{(n)})=J_n, {\bf M}^{(n)}={\bf M}^{(n)T}, {\bf Z}^{(n)}\geq 0, \mathbf{0}\preceq {\bf M}^{(n)}\preceq \mathbf{I}.
\end{split}
\end{equation}
We apply the alternating direction method of multipliers to solve (\ref{cmd1}). The augmented Lagrangian function of (\ref{cmd1}) is given by
\begin{equation}\label{La}
\begin{split}
    & \ L(\mathbf{K}^{(n)},\mathbf{X}^{(n)},\mathbf{Z}^{(n)},\mathbf{M}^{(n)})\\
    =& \ \frac{1}{2}\| \mathbf{A}_{(n)}- \mathbf{K}^{(n)}\mathbf{A}_{(n)}\|^2_F+\theta\| {\bf X}^{(n)} \|_1+\delta_{\mathbb{R}^{I_n\times I_n}_+}(\mathbf{Z}^{(n)})
  \ -\langle\mathbf{\Lambda}^{(n)}_1,\mathbf{K}^{(n)}-\mathbf{X}^{(n)}\rangle-\langle\mathbf{\Lambda}^{(n)}_2, \mathbf{K}^{(n)}-\mathbf{Z}^{(n)}\rangle \\
   &\ -\langle\mathbf{\Lambda}^{(n)}_3,\mathbf{K}^{(n)}
-\mathbf{M}^{(n)}\rangle+\frac{\rho^{(n)}_1}{2}\|\mathbf{K}^{(n)}-\mathbf{X}^{(n)}\|^2_F
  \  +\frac{\rho^{(n)}_2}{2}\|\mathbf{K}^{(n)}
  -\mathbf{Z}^{(n)}\|^2_F + \frac{\rho^{(n)}_3}{2}\|\mathbf{K}^{(n)}-\mathbf{M}^{(n)}\|^2_F,
\end{split}
\end{equation}
where  $\delta_{\mathbb{R}_+^{I_n\times I_n}}$ denotes the indicator of $\mathbb{R}_+^{I_n\times I_n}$, i.e.,
$$
\delta_{\mathbb{R}_+^{I_n\times I_n}}(\mathbf{X}):=
\left\{\begin{array}{cll}
0, & \mbox{if} \ \mathbf{X}\in \mathbb{R}_+^{I_n\times I_n},\\
+\infty, & \mbox{otherwise}.\\
\end{array}\right.
$$
The iterative system of ADMM  is given as follows:
\begin{equation}\label{Sys}
\begin{split}
(\mathbf{K}^{(n)})_{i+1} &=\mathop{\arg\min}\big\{L(\mathbf{K},(\mathbf{X}^{(n)})_i,(\mathbf{Z}^{(n)})_i,(\mathbf{M}^{(n)})_i): tr(\mathbf{K}^{(n)})=J_n\big\},\\
((\mathbf{X}^{(n)})_{i+1},(\mathbf{Z}^{(n)})_{i+1},(\mathbf{M}^{(n)})_{i+1})&=\mathop{\arg\min}\big\{L((\mathbf{K}^{(n)})_{i+1},\mathbf{X}^{(n)},\mathbf{Z}^{(n)},\mathbf{M}^{(n)}):
\mathbf{0}\preceq \mathbf{M}^{(n)}\preceq \mathbf{I}, \mathbf{M}^{(n)T}=\mathbf{M}^{(n)}\big\}, \\
(\mathbf{\Lambda}^{(n)}_1) _{i+1}&= (\mathbf{\Lambda}^{(n)}_1)_{i}-\gamma^{(n)}\rho^{(n)}_1((\mathbf{K}^{(n)})_{i+1}-(\mathbf{X}^{(n)})_{i+1}),\\
(\mathbf{\Lambda}^{(n)}_2)_{i+1} &= (\mathbf{\Lambda}^{(n)}_2)_{i}-\gamma^{(n)}\rho^{(n)}_2((\mathbf{K}^{(n)})_{i+1}-(\mathbf{Z}^{(n)})_{i+1}), \\
 (\mathbf{\Lambda}^{(n)}_3)_{i+1}&= (\mathbf{\Lambda}^{(n)}_3)_{i}-\gamma^{(n)}\rho^{(n)}_3((\mathbf{K}^{(n)})_{i+1}-(\mathbf{M}^{(n)})_{i+1}),
\end{split}
\end{equation}
where $\gamma^{(n)}\in (0,(1+\sqrt{5})/2)$.

\subsubsection{The Computation of ${\bf K}^{(n)}$}
For $\mathbf{K}^{(n)}$, by \cite{Lai15}, we can solve it as
\begin{equation}\label{K}
   (\mathbf{ K}^{(n)})_{i+1}=(\mathbf{B}^{(n)})_i-\frac{tr((\mathbf{B}^{(n)})_i)-J_n}{I_n}\mathbf{I},
\end{equation}
where
\begin{equation*}
\begin{split}
&(\mathbf{B}^{(n)})_i=(\mathbf{P}+\mathbf{Q})\Big(\mathbf{A}_{(n)}\mathbf{A}_{(n)}^T+(\rho^{(n)}_1+\rho^{(n)}_2+\rho^{(n)}_3) \mathbf{I}\Big)^{-1}
\end{split}
\end{equation*}
with
\[
\begin{split}
\mathbf{P}=\mathbf{A}_{(n)}\mathbf{A}_{(n)}^T+(\mathbf{\Lambda}^{(n)}_1) _{i}+(\mathbf{\Lambda}^{(n)}_2) _{i}+(\mathbf{\Lambda}^{(n)}_3)_{i},\quad
\mathbf{Q}=\rho^{(n)}_1(\mathbf{X}^{(n)})_i+
\rho^{(n)}_2(\mathbf{Z}^{(n)})_i+\rho^{(n)}_3(\mathbf{M}^{(n)})_i.
\end{split}
\]

\subsubsection{The Computation of ${\bf X}^{(n)}$}
For $\mathbf{X}^{(n)}$, it is the shrinkage
\begin{equation}\label{X1}
\begin{split}
(\mathbf{X}^{(n)}&)_{i+1}
=\arg\min_{\mathbf{X}^{(n)}}\Big\{\theta\|\mathbf{X}^{(n)}\|_1+\frac{\rho^{(n)}_1}{2}\|\mathbf{X}^{(n)}-((\mathbf{K}^{(n)})_{i+1}-\frac{1}{\rho^{(n)}_1}(\mathbf{\Lambda}^{(n)}_1)_i)\|^2_F\Big\}.
\end{split}
\end{equation}
Thus,
\begin{equation}\label{X2}
    (\mathbf{X}^{(n)})_{i+1}=\mbox{Shrinkage}\Big((\mathbf{K}^{(n)})_{i+1}-\frac{1}{\rho^{(n)}_1}
    (\mathbf{\Lambda}^{(n)}_1)_i,\frac{\theta}{\rho^{(n)}_1}\Big),
\end{equation}
where $\mbox{Shrinkage}(x,\tau):= \mbox{sign}(x)\max\{|x|-\tau,0\}$ and  $\mbox{sign}(\cdot)$ denotes the signum function, i.e.,
$$
\mbox{sign}(x):= \left\{\begin{array}{cl}
1, & \mbox{if} \ x>0,\\
0, & \mbox{if} \ x=0,\\
-1, & \mbox{if} \  x<0.\\
\end{array}\right.
$$

\subsubsection{The Computation of ${\bf Z}^{(n)}$}
For $\mathbf{Z}^{(n)}$, it is the projection onto $\mathbb{R}^{I_n\times I_n}_+$,
\begin{equation}\label{Z1}
\begin{split}
(\mathbf{Z}^{(n)}&)_{i+1}
=\arg\min_{\mathbf{Z}^{(n)}}\Big\{\delta_{\mathbb{R}^{I_n\times I_n}_{+}}(\mathbf{Z}^{(n)})+\frac{\rho^{(n)}_2}{2}\Big\|\mathbf{Z}^{(n)}-\Big((\mathbf{K}^{(n)})_{i+1}-\frac{1}{\rho^{(n)}_2}(\mathbf{\Lambda}^{(n)}_2)_i\Big)\Big\|^2_F\Big\}.
\end{split}
\end{equation}
It is given by
\begin{equation}\label{Z2}
    (\mathbf{Z}^{(n)})_{i+1}=\Pi_{\mathbb{R}^{I_n\times I_n}_{+}}\Big((\mathbf{K}^{(n)})_{i+1}-\frac{1}{\rho^{(n)}_2}(\mathbf{\Lambda}^{(n)}_2)_i\Big),
\end{equation}
where $\Pi_{\mathbb{R}^{I_n\times I_n}_{+}}$ is the projection onto $\mathbb{R}^{I_n\times I_n}_{+}$.

\subsubsection{The Computation of ${\bf M}^{(n)}$}
For $\mathbf{M}^{(n)}$, it is
\begin{equation}\label{M1}
\begin{split}
(\mathbf{M}^{(n)})_{i+1}=\mathop{\arg\min}\Big\{\frac{\rho^{(n)}_3}{2}\Big\|\mathbf{M}^{(n)}
-((\mathbf{K}^{(n)})_{i+1}-\frac{1}{\rho^{(n)}_3}(\mathbf{\Lambda}^{(n)}_3)_i)\Big\|^2_F:~0\preceq \mathbf{M}^{(n)} \preceq I,  \mathbf{M}^{(n)T}=\mathbf{M}^{(n)}\Big\}.
\end{split}
\end{equation}
By the projection of a matrix on the symmetric positive matrix \cite[Section 4.3]{Qi06}  and \cite[Lemma 2.1]{Tsen98}, we have
\begin{equation}\label{M2}
     (\mathbf{M}^{(n)})_{i+1}=\frac{1}{2}\tilde{\mathbf{V}}^{(n)}\min\{\max\{\tilde{\Sigma}^{(n)},0\},1\}\tilde{\mathbf{V}}^{(n)T},
\end{equation}
where $\tilde{\mathbf{V}}^{(n)}$, $\tilde{\mathbf{\Sigma}}^{(n)}$ are the eigenvalue decomposition of
$$
(\mathbf{K}^{(n)})_{i+1}-\frac{1}{\rho^{(n)}_3}(\mathbf{\Lambda}^{(n)}_3)_i
+\Big((\mathbf{K}^{(n)})_{i+1}-\frac{1}{\rho^{(n)}_3}(\mathbf{\Lambda}^{(n)}_3)_i\Big)^T.
$$

\subsubsection{The Algorithm and Convergence Analysis}
The proposed algorithm for $\{\mathbf{K}^{(n)}\}^d_{n=1}$ is given in Algorithm 1. For each $\mathbf{K}^{(n)}$ with $n=1,2,\cdots, d$, in the alternating direction method of multipliers, two blocks of variables $\mathbf{K}^{(n)}$ and
$(\mathbf{X}^{(n)};\mathbf{Z}^{(n)};\mathbf{M}^{(n)})$ are updated in each iteration in our algorithm. The convergence of the algorithm can be guaranteed, see (\cite{Gaba76,boyd2011distributed,Lai15}).  The detailed proof is given in Appendix, we follow the main idea from \cite{Lai15}, the difference is our proof is based on our model which has one more constraint than theirs.

For $n$-th mode when $n=1,2,\ldots, d$,
\begin{theorem}
Assume that  $\gamma^{(n)}\in (0,(1+\sqrt{5})/2)$,
for any $\rho^{(n)}_1, \rho^{(n)}_2, \rho^{(n)}_3 > 0$, the iterative sequence $((\mathbf{K}^{(n)})_i;(\mathbf{X}^{(n)})_i;(\mathbf{Z}^{(n)})_i;(\mathbf{M}^{(n)})_i;(\mathbf{\Lambda}^{(n)}_1)_i;(\mathbf{\Lambda}^{(n)}_2)_i;(\mathbf{\Lambda}^{(n)}_3)_i)$ generated by Algorithm 1 from any initial point converges to $((\mathbf{K}^{(n)})_*;(\mathbf{X}^{(n)})_*;(\mathbf{Z}^{(n)})_*;(\mathbf{M}^{(n)})_*;(\mathbf{\Lambda}^{(n)}_1)_*;(\mathbf{\Lambda}^{(n)}_2)_*;(\mathbf{\Lambda}^{(n)}_3)_*)$, where $((\mathbf{K}^{(n)})_*;(\mathbf{X}^{(n)})_*;(\mathbf{Z}^{(n)})_*;(\mathbf{M}^{(n)})_*)$ is a solution of (\ref{md3}).
\end{theorem}
After $\mathbf{K}^{(n)}$ is computed by Algorithm 1, $\mathbf{U}^{(n)}$ can be recovered based on Lemma 2, i.e., we compute the $J_n$ eigenvectors corresponding to $J_n$ largest eigenvalues of $\mathbf{K}^{(n)}$, then use hard clustering evaluation \cite{Ding06} to form ${\bf U}^{(n)}$. Or one can simply use some known clustering methods \cite{Hart75,Luxb07, Macq67} such as $k$-means, spectral clustering to get $\mathbf{U}^{(n)}$ from $\mathbf{K}^{(n)}$.

\begin{algorithm}[h]
\caption{Alternating direction method of multipliers for model ONTD}
\begin{algorithmic}[1]
\REQUIRE Given $\mathcal{A} \in\mathbb{R}_{+}^{I_1\times I_2\times \cdots \times I_d}$, $(J_1,J_2,\cdots,J_d)$, the parameters $\theta$,  $\rho_1$, $\rho_2$, $\rho_3$, $\gamma$, initial values $({\bf K}^{(n)})_1\in\mathbb{R}_{+}^{I_n\times I_n}$, $({\bf X}^{(n})_1\in\mathbb{R}^{I_n\times I_n}_{+}$, $({\bf Z}^{(n)})_{1}\in\mathbb{R}^{I_n\times I_n}_{+}$, $({\bf M}^{(n)})_{1}\in\mathbb{R}^{I_n\times I_n}_{+}$, $({\bf \Lambda}^{(n)}_1)_{1}\in\mathbb{R}^{I_n\times I_n}$, $({\bf \Lambda}^{(n)}_2)_{1}\in\mathbb{R}^{I_n\times I_n}$, $({\bf \Lambda}^{(n)}_3)_{1}\in\mathbb{R}^{I_n\times I_n}$, and the stopping criterion $\epsilon$.
\ENSURE ${\bf K}$
\STATE \textbf{Step 0}. Unfolding tensor $\mathcal{A}$ from $n$ mode, obtain the $n$-mode matrix $\mathbf{A}_{(n)}$.
\STATE \textbf{Step 1}. Compute $(\mathbf{K}^{(n)})_{i+1}$ by (\ref{K}).
\STATE \textbf{Step 2}. Compute $(\mathbf{X}^{(n)})_{i+1}$, $(\mathbf{Z}^{(n)})_{i+1}$, $(\mathbf{M}^{(n)})_{i+1}$ by (\ref{X2}), (\ref{Z2}), (\ref{M2}), respectively.
\STATE \textbf{Step 3}. Update $(\mathbf{\Lambda}^{(n)}_1)_{i+1}$, $(\mathbf{\Lambda}^{(n)}_2)_{i+1}$, $(\mathbf{\Lambda}^{(n)}_3)_{i+1}$ by (\ref{Sys}).
\STATE \textbf{Step 4}. If the termination criterion is not met, go to Step 1.
\end{algorithmic}
\end{algorithm}

\subsection{The Core Tensor ${\cal S}$}

After we get the factor matrices $\{\mathbf{U}^{(n)}\}^d_{n=1}$, the core tensor $\mathcal{S}$ can be obtained by
\begin{equation}\label{core}
\begin{split}
\mathcal{S}=\mathop{\arg\min}\big\{&\|\mathcal{A}-\mathcal{S}\times_1 \mathbf{U}^{(1)}\times_2 \mathbf{U}^{(2)}\cdots \times_d \mathbf{U}^{(d)}\|^2_F: \mathcal{S}\in \mathbb{R}^{J_1\times J_2\times\cdots\times_d J_d}_+\big\}
\end{split}
\end{equation}
Here, we use $n$-th mode $\mathbf{S}_{(n)}$ to fold into the core tensor $\mathcal{S}$.
We remark that (\ref{e5}) is the unfolding version of (\ref{core}), the solution of (\ref{e5}) is the same as that of (\ref{core}).
It is valid to choose any value of $n$ in between 1 and $d$.

It is clear that $\mathbf{S}_{(n)}$ can be solved by the following least square problem:
\begin{equation}\label{e5}
\begin{split}
\mathbf{S}_{(n)}=\mathop{\arg\min}\big\{&\|\mathbf{A}_{(n)}-\mathbf{U}^{(n)}\mathbf{S}_{(n)}(\mathbf{U}^{(n+1)}\otimes \mathbf{U}^{(n+2)}\cdots\ \ \cdots \otimes\mathbf{U}^{(d)}\otimes\mathbf{U}^{(1)}\otimes \cdots \otimes\mathbf{U}^{(n-1)})^{T}\|^2_F: \\
& \ \mathbf{S}_{(n)}\in \mathbb{R}^{J_n\times J_{n+1}\cdots J_d J_1\cdots J_{n-1}}_+\big\}.
\end{split}
\end{equation}
We remark that (\ref{e5}) is valid for $1 \le n \le d$.
Let $\mathbf{U}^{(\backslash n)}=\mathbf{U}^{(n+1)}\otimes \mathbf{U}^{(n+2)}\cdots
     \cdots \otimes\mathbf{U}^{(d)}\otimes\mathbf{U}^{(1)}\otimes \cdots \otimes\mathbf{U}^{(n-1)}$ and take use of the property of Kronecker product that $vec(\mathbf{U}^{(n)}\mathbf{S}_{(n)}\mathbf{U}^{(\backslash n) T})=(\mathbf{U}^{(\backslash n)}\otimes\mathbf{U}^{(n)})vec(\mathbf{S}_{(n)})$, (\ref{e5}) can be written as follows,
\begin{equation}\label{core1}
\begin{split}
 \mathop{\min}_{\mathbf{S}_{(n)}} \ \|vec(\mathbf{A}_{(n)})-(\mathbf{U}^{(\backslash n)}\otimes\mathbf{U}^{(n)})vec(\mathbf{S}_{(n)})\|^2_F \quad
 \mbox{s.t.} \quad  \ \mathbf{S}_{(n)}\in \mathbb{R}^{J_n\times J_{n+1}\cdots J_d J_1\cdots J_{n-1}}_+.
 \end{split}
\end{equation}
Denote $(\mathbf{U}^{(\backslash n)}\otimes\mathbf{U}^{(n)})$ as $\mathbf{L}$, the least square solution of (\ref{core1}) can be obtained by
\begin{equation}\label{S}
  vec(\mathbf{S}_{(n)})=\Pi_{\mathbb{R}^{I_n\times I_n}_{+}}\Big((\mathbf{L}^T\mathbf{L})^{-1}\mathbf{L}^Tvec(\mathbf{A}_{(n)})\Big).
\end{equation}
Core tensor $\mathcal{S}$  can be obtained by folding $\mathbf{S}_{(n)}$ from $n$-th mode.

Note that if $\{\mathbf{U}^{(n)}\}^d_{n=1}$ are column orthogonal, then
$$
\mathcal{S}=\mathcal{A}\times_1 \mathbf{U}^{(1)T}\times_2 \mathbf{U}^{(2)T}\cdots \times_d \mathbf{U}^{(d)T}.
$$

\section{Experimental Results}

In this section, we conduct experiments on three applied areas to test the performance of the proposed ONTD model. All experiments are run on Intel(R) Core(TM) i7-5600 CPU @2.60GHZ with 8GB of RAM.

As comparison, we use some general methods in the experiments, for example, the tensor methods such as NTD\cite{Kim07}, HOOI\cite{DeLa00best} and matrix methods like  NMF\cite{Lee99, Berr07}, SN-ONMF\cite{Pan18}.

\subsection{Feature Extraction and Face Recognition}
In this subsection, we use face ORL database\cite{Sam94} which contains 400 images of 40 individuals. The images are captured at different times and different variations including facial details and expression. These images are in gray scale and normalized to the resolution of $112\times 92$ pixels. We randomly select $p\times 100\%$ sample images for each person to be the training data, and the others are used for testing. Use the tensor methods, we decompose the training data with setting $(J_1, J_2)$. The slices of $\mathcal{S}\times_1 \mathbf{U}^{(1)}\times_2 \mathbf{U}^{(2)}$ can be seen as basis images. For comparison, matrix methods are also used to learn the basis matrix from the flattened training data with the same number of features. Once we get the basis matrix from the training data, the nonnegative projection of a new test sample onto each basis matrix is used as the features for recognition. The KNN classifier is used for recognition by using the extracted features, here the distance is measured by their correlations.

First we use the basis images based on $20\%$ training data set (80 images). The number of features is setting to $J_3$. To show reconstruction capacity for original image, we define compression ratio
$$
compression~~ratio=\frac{Compressed~~size}{Original~~size}\times 100\%
$$
and the average reconstruction error
$$
avg.error=\frac{1}{N}\sum^N_{k=1}\frac{\|I^k_{original}-I^k_{recon}\|_F}{\|I^k_{original}\|_F},
$$
where $I^{k}_{original}$ represents the $k$-th original image, $I^{k}_{recon}$ represents the $k$-th reconstruction image, $N$ is the number of the images.

In Table \ref{1}, for tensor methods, we summarize the average reconstruction error
according to different $(J_1,J_2,J_3)$, here $(J_1,J_2)=(5,5)$, $(10,10)$, $(15, 15)$, $(20,20)$, $(25,25)$, $(30,30)$, $(40,40)$, $(60,60)$. It can be seen from Table \ref{1} that the reconstruction error  decrease when we increase $J_3$.  From Table \ref{1},  the reconstruction error of ONTD and HOOI are similar, which is less than NTD.
\begin{table*}[!t]
  \centering
    \caption{The face-reconstruct average error of different methods with different number of features.}\label{1}
  \begin{tabular}{|c|c|c|c|c|c|c|c|c|}
  \hline
   ~~ & $J_3=5$ & $J_3=10$  & $J_3=15$ & $J_3=20$ & $J_3=25$& $J_3=30$ & $J_3=40$ & $J_3=60$ \\
    \hline
  ONTD&0.0029&0.0025&0.0023&0.0022&0.0021&0.0020&0.0019&0.0017\\
   \hline
  HOOI&0.0029&0.0025&0.0023&0.0022&0.0021&0.0020&0.0019&0.0017\\
   \hline
   NTD&0.0030&0.0027&0.0026&0.0025&0.0024&0.0024&0.0023&0.0022\\
  \hline
\end{tabular}
\end{table*}

In the following, we set $(J_1,J_2)$ to be $(15,15)$ and get the basic matrix from different training data set. Test the the recognition accuracy on the rest data set (i.e., testing set) according to the basic matrix. We show the recognition results on Table \ref{2}-Table \ref{5}.  It can be seen that the number of features ($J_3$) affects the recognition. A Large value of $J_3$ often leads to a higher accuracy, but when $J_3$ is too large, the accuracy decrease. We also notice that the recognition accuracy can be increase when using more training data. From Table \ref{2}-Table \ref{5}, the accuracy obtained by our method (ONTD) is higher than the other methods in most cases and ONTD  takes much less time than the other tensor methods.

\begin{table*}[!t]
  \centering
  \caption{The face-reconstruct results of different methods by using 20\% as training set.}\label{2}
  \begin{tabular}{|c|c|c|c|c|c|c|c|c|}
  \hline
  \multirow{2}{*}& \multicolumn{2}{|c|}{$J_3$=20} & \multicolumn{2}{|c|}{ $J_3$=40}  & \multicolumn{2}{|c|}{$J_3$=60}& \multicolumn{2}{|c|}{$J_3$=80}\\
 \cline{2-9}&accuracy&time&accuracy&time&accuracy&time&accuracy&time\\
 \hline
 ONTD&\textbf{0.7750}&4.33&\textbf{0.8063}&5.93&\textbf{0.8063}&3.63&0.7531&5.24\\
 \hline
 HOOI&0.7250&146.37&0.7750&309.97&0.7750&434.96&0.7063&350.29\\
 \hline
 NTD&0.3000& 121.35&0.3500&130.94&0.3594& 146.22&0.3875&153.57 \\
 \hline
 NMF& 0.7031&150.82&0.7156&176.50&0.7000 &211.96& 0.7281&273.69\\
 \hline
 SN-ONMF&0.7125& 2.04&0.7438&2.06&0.6844&1.94&\textbf{0.7688}& 1.97 \\
 \hline
 \end{tabular}
\end{table*}

\begin{table*}[!t]
  \centering
  \caption{The face-reconstruct results of different methods by using 30\% as training set.}\label{3}
  \begin{tabular}{|c|c|c|c|c|c|c|c|c|}
  \hline
  \multirow{2}{*}& \multicolumn{2}{|c|}{$J_3$=20} & \multicolumn{2}{|c|}{ $J_3$=40}  & \multicolumn{2}{|c|}{$J_3$=60}& \multicolumn{2}{|c|}{$J_3$=80}\\
 \cline{2-9}&accuracy&time&accuracy&time&accuracy&time&accuracy&time\\
 \hline
 ONTD&  \textbf{0.8000}&3.92&\textbf{0.8393}&5.10&\textbf{0.8321}&6.42&\textbf{0.8250}&6.70\\
 \hline
 HOOI&0.7536&125.22&0.8143 &182.70&0.8036&207.25&0.7321&387.94\\
 \hline
 NTD&0.4750& 149.29& 0.4571&157.66& 0.4571 & 163.61& 0.4321 & 214.32 \\
 \hline
 NMF& 0.7357&219.93& 0.7893&256.74&0.7536&291.91& 0.7321&372.42\\
 \hline
 SN-ONMF&0.7321& 3.99&0.7857&4.06&0.7321&4.11&0.6500& 4.24 \\
 \hline
 \end{tabular}
\end{table*}

\begin{table*}[!t]
  \centering
  \caption{The face-reconstruct results of different methods via using 40\% as training set.}\label{4}
  \begin{tabular}{|c|c|c|c|c|c|c|c|c|}
  \hline
  \multirow{2}{*}& \multicolumn{2}{|c|}{$J_3$=20} & \multicolumn{2}{|c|}{ $J_3$=40}  & \multicolumn{2}{|c|}{$J_3$=60}& \multicolumn{2}{|c|}{$J_3$=80}\\
 \cline{2-9}&accuracy&time&accuracy&time&accuracy&time&accuracy&time\\
 \hline
 ONTD&  \textbf{0.8750}&10.06&\textbf{0.8750}&9.12&\textbf{0.8875}&9.84& \textbf{0.8958}&8.46\\
 \hline
 HOOI&0.8083& 186.10&0.8375  &309.56&0.8500& 416.10&0.8500&542.87\\
 \hline
 NTD&0.4417&215.10& 0.4667&220.56& 0.4833 & 267.85&0.5542 & 303.57 \\
 \hline
 NMF&  0.7667&376.67& 0.8333 &420.3&0.8250&457.10&  0.7833&502.94\\
 \hline
 SN-ONMF&0.8208& 6.53&0.8208&7.42&0.7917&6.83&0.7500& 6.46 \\
 \hline
 \end{tabular}
\end{table*}

\begin{table*}[!t]
  \centering
  \caption{The face-reconstruct results of different methods by using 50\% as training set.}\label{5}
  \begin{tabular}{|c|c|c|c|c|c|c|c|c|}
  \hline
  \multirow{2}{*}& \multicolumn{2}{|c|}{$J_3$=20} & \multicolumn{2}{|c|}{ $J_3$=40}  & \multicolumn{2}{|c|}{$J_3$=60}& \multicolumn{2}{|c|}{$J_3$=80}\\
 \cline{2-9}&accuracy&time&accuracy&time&accuracy&time&accuracy&time\\
 \hline
 ONTD&  \textbf{0.8800}&12.24&\textbf{0.8800}&12.91&\textbf{0.8900}&5.74& \textbf{0.8800}&11.60\\
 \hline
 HOOI&0.8200& 186.47&0.8400  &296.72&0.8600 & 412.33& 0.8100 &552.86\\
 \hline
 NTD&0.4450&296.75& 0.5250&328.19& 0.4300 & 347.22&0.4950 & 375.75 \\
 \hline
 NMF&  0.8050&493.78&0.8150 &470.28& 0.7950&586.04&  0.8050&684.95\\
 \hline
 SN-ONMF& 0.8350& 9.68&0.8400&9.81&0.8250&9.92& 0.7600& 10.17 \\
 \hline
 \end{tabular}
\end{table*}


\subsection{Image Representation}
In the following data sets, the image sets are belong to several different classes, each class is represented by a 3-order tensor. For example, there are $N$ tensor objects size of $I_1\times I_2\times I_3$ that can be classified into $r$ classes. To implement the tensor methods, we concatenate these $N$ 3-order tensor objects to a $I_1\times I_2\times I_3\times N$ tensor. We follow the idea in \cite{Lixu17,Phan10}, i.e., given approximate rank $(J_1,J_2,J_3)$,  use the tensor methods to find three common  projection matrices $\mathbf{U}^{(1)}$, $\mathbf{U}^{(2)}$, $\mathbf{U}^{(3)}$ for these 3-order tensor objects. These tensor objects then can be represented as $N$ corresponding core tensors of size $ (J_1,J_2,J_3)$. When use the matrix methods, we vectorize $N$ tensor objects and get a $I_1I_2I_3\times N$ matrix first.

To test the performance of the proposed model, we reduce the dimension of tensor objects and report the space savings.
$$
Space~~savings=1-\frac{Compressed~~size}{Original~~size}
$$

Since the data sets have ground truth class label, We also classify these $N$ reduced tensor objects (core tensor) with nearest neighborhood classifier and use the leave-one-out scheme, then show the average precision to evaluate the classification performance.
\subsubsection{The ORL Database of Faces}
Here we use ORL database of faces again. As we know, there are ten different images of each of 40 distinct individual. Each image is of $92\times 112$ pixels. We first uniformly resized them into $20 \times 20$, then randomly choose two images from ten for each subjects to form a $20 \times 20 \times 2$ tensor. Therefore, for each person, we get five tensor objects. From 40 distinct individuals, we can obtain 200 tensor objects totally. We select 40\% tensor objects from 200 tensor objects randomly to form a $20 \times 20 \times 2 \times 80$ training tensor to tune a good approximate rank. The rest 120 tensor objects are used to form the test data set to evaluate the results.

Firstly we set the approximate rank to be $[J_1,J_2,J_3]$ for the tensor methods, and apply these methods on the training set, we get the common matrices $(\mathbf{U}^{(1)}, \mathbf{U}^{(2)}, \mathbf{U}^{(3)})$ and 80 core tensors of $[J_1,J_2,J_3]$. For matrix methods, we unfold the tensors to a $800 \times 80$ matrix. The approximate rank for matrix is set to be $r$. We compute the best parameters $(J_1,J_2,J_3)$ to get the highest precision by using the the tensor methods, or the best parameter $r$ to get its largest precision by each matrix method. Then the parameters $(J_1,J_2,J_3)$ or $r$ are used in the test set. We show the results in Table \ref{6}.

 From Table \ref{6}, compare to other four methods, ONTD has the highest precision and the least time.

\begin{table}[htbp]
  \centering
    \caption{The clustering  results for ORL database.}\label{6}
  \begin{tabular}{|c|c|c|c|c|}
  \hline
  ~& Best~parameter & precision & Space saving & time\\
  \hline
ONTD &$[9,3,2]$ & 0.1375&  0.9300 & 0.03 \\
  \hline
HOOI &$[3,3,2]$&0.125&   0.9762 & 18.79\\
\hline
NTD &$[7,3,2]$& 0.0917&     0.9454 &  5.76\\
     \hline
NMF &15 & 0.125     &   0.8562   &    0.15\\
     \hline
SN-ONMF &25 &0.1 &      0.7604&    6.366\\
     \hline
\end{tabular}
\end{table}
\subsubsection{Cambridge Gesture database}
Cambridge Gesture data sets contains 900 images from nine hand gesture classes\cite{KimTK09}. Each class has 100 image sequences (5 illuminations $\times$ 10 arbitrary motions $\times$ 2 subjects).  Each image has $320\times 240$ pixels. For all video sequences, we use the gray scale representation and  uniformly resized them  into $20 \times 20 \times 32$. We get 900 tensor objects of size $20\times 20\times 32$. We select $30\%$ tensor objects from 900 tensor objects randomly, i.e., $270$ tensor objects of size $20\times 20\times 32$. From 270 tensor objects, we randomly use 90 tensor objects to form a $20\times 20\times 32\times 90$ training tensor to tune a good approximate rank for all methods, a $20\times 20\times 32\times 180$  tensor which is formed by the rest 180 tensor objects, are utilized as the test set to evaluate the results.

We set the approximate rank to be $(J_1,J_2,J_3)$ for all the tensor methods and apply them on the training set, we get the common matrices $(\mathbf{U}^{(1)}, \mathbf{U}^{(2)}, \mathbf{U}^{(3)})$ and 90 core tensors of $(J_1,J_2,J_3)$. For matrix methods, similarly, we unfold the tensors to a $12800\times 90$ matrix. The approximate rank for matrix is set to be $r$. We tune the best parameters, i.e., approximate rank ($J_1,J_2,J_3$) or rank $r$ by training set. The best parameters are then used in the test set.

 The results of test set are listed in Table \ref{7}. From space savings, we find that matrix methods needs more storage than tensor methods, it implies that tensor methods performs better in the data reduction. It is very important especially when data set is extreme large. Note that the precision of ONTD is the largest and the time is the least of all methods.

\begin{table}[htbp]
  \centering
  \caption{The clustering  results for Cambridge database.}\label{7}
  \begin{tabular}{|c|c|c|c|c|}
  \hline
  ~& Best~parameter & precision & Space saving & time\\
  \hline
ONTD &$[7,7,3]$&0.8889& 0.9884&0.04\\
  \hline
HOOI &$[7,7,2]$&0.8667& 0.9922&13.64\\
\hline
NTD &$[8,7,3]$&0.6333& 0.9867&93.14\\
     \hline
NMF &10 &  0.7556&  0.9437 &3.97\\
     \hline
SN-ONMF &25 & 0.6444 &  0.8592 &12.64\\
     \hline
\end{tabular}
\end{table}

\subsubsection{KTH Human Action database}
KTH database \cite{Schu04} contains six types of human actions, including walking, running, jogging, boxing, hand clapping and hand waving. There are totally 600 videos and each type has 100 videos.  The action videos are performed several times by 25 subjects in 4 different scenario (outdoors, outdoors with scale variation, outdoors with different clothes and indoors). The resolution of video frames is $160\times 120$ pixels. The frames for the first time action of each video are extracted as our video data. To standardize the length of videos, we sample 64 frames from each video and take the middle 32 frames. For each video, we use grayscale representation and resize into $20\times 20\times 32$. We get 600 tensor objects whose size is $20\times 20\times 32$, then we select 200 tensor objects randomly. Of 200 selected tensor objects, 100 is used to form a $20\times 20\times 32\times 100$ training tensor, the rest 100 tensor objects are utilized as the test set to show the performance of these methods.

For each tensor object, we set the size of the core tensor to be $[J_1,J_2,J_3]$, and apply the methods on the training set. For matrix methods, similarly, we unfold the tensor to be a $12800\times 100$ matrix and the approximate rank is set to be $r$. We tune the best approximate rank  for each method that computed on the training set to get the largest precision. Then use these parameters on test set to test the performance of each method. From Table \ref{8}, the precision of ONTD is the largest of all methods and the time cost is the least too. From space saving, it has a very similar results for each tensor method. All tensor methods perform much better than matrix methods in data reduction.

\begin{table}[!t]
  \centering
  \caption{The clustering  results for KTH database.}\label{8}
  \begin{tabular}{|c|c|c|c|c|}
  \hline
  ~& Best~parameter & precision   & Space saving & time\\
  \hline
ONTD &$[3,6,2]$ & 0.4400&  0.9970 & 0.02 \\
  \hline
HOOI &$[2,2,2]$&0.3500&      0.9993 &  3.94\\
\hline
NTD &$[2,3,5]$& 0.3000 &       0.9975 &   51.52\\
     \hline
NMF &3 & 0.3800      &   0.9698   &    0.53\\
     \hline
SN-ONMF &16 &0.4200 &        0.8388 &      2.80\\
     \hline
\end{tabular}
\end{table}

\subsubsection{The MNIST database}
The MNIST database of handwritten digits has a training set of 60000 examples, and a test set of 10000 examples. There are ten different patterns from 0 to 9. Each image is $28\times 28$ pixels. Of the training data set, we randomly choose 90\% images of each pattern, i.e. 5400 images,  to form 54 tensor objects whose size is $28 \times 28 \times 100$ tensor. We obtain 540 tensor objects totally from 10 different patterns. Therefore the size of the training tensor is $28 \times 28 \times 100\times 540$. For test data, similarly, from test set, we select 80\% images of each pattern, i.e. 800 images, to form 8 tensor objects whose size is $28 \times 28 \times 100$. Thus, there are totally 80 tensor objects from 10 patterns. We get the testing tensor sizes of $28 \times 28 \times 100 \times 80$ finally.

We set the approximate rank to be $[J_1,J_2,J_3]$ for all the tensor methods and $r$ for all the matrix methods, then apply these methods on the training data. We tune best parameters approximate rank ($J_1,J_2,J_3$) or rank $r$. Then the best parameter are applied in the test data. In Table \ref{9}, we list the results of test data. It is clearly that all tensor methods perform much better than matrix methods in data reduction.  The precision of the methods are all very high, and ONTD takes much less time than other tensor methods.


\begin{table}[!t]
  \centering
  \caption{The clustering  results for MNIST database.}\label{9}
  \begin{tabular}{|c|c|c|c|c|}
  \hline
  ~& Best~parameter & precision  & Space saving & time\\
  \hline
ONTD &$[2,2,4]$ & 1&  0.9997 &0.31 \\
  \hline
HOOI &$[2,2,3]$&1&  0.9998 &  103.68\\
\hline
NTD &$[2,2,4]$&1&   0.9997  &   656.86\\
     \hline
NMF &5 & 0.9875    &   0.9374   &    17.17\\
     \hline
SN-ONMF &6 &1&     0.9249  &     0.18\\
     \hline
\end{tabular}
\end{table}

\subsection{Hyperspectral Unmixing}

The hyperspectral imaging collects information from the object by taking at different wavelengths. The images are obtained by measuring the percentage of the
light hitting a material which is called reflectance. Like other spectral imaging, the purpose of hyperspectral imaging is to find objects, identify materials or
detect processes. It therefore has wide applications in agriculture, mineralogy, physics, environment and many other fields. Hyperspectral unmixing aims to
classify the pixels to different clusters, with each corresponding to a material.

We apply the partial ONTD model on hyperspectral unmixing, precisely, given an $I_1\times I_2\times I_3$ nonnegative tensor $\mathcal{A}$ and a factorization rank $r$, solve
\begin{equation*}
\begin{split}
\min_{\mathbf{U}^{(1)},\mathcal{S}}
& \ \|\mathcal{A}-\mathcal{S}\times_1 \mathbf{U}^{(1)}\times_2 \mathbf{I}\times_3 \mathbf{I}\|^2_F \\
\mbox{s.t.}& \ \mathcal{S}\geq 0,~\mathbf{U}^{(1)}\geq 0,~\mathbf{U}^{(1)T}\mathbf{U}^{(1)}=\mathbf{I},\\
&\ \mathbf{U}^{(1)}\in \mathbb{R}^{I_1\times r},\mathcal{S}\in \mathbb{R}^{r\times I_2\times I_3}.
\end{split}
\end{equation*}

In this subsection, we use the Samson data set \cite{Zhu14spe}. In the image, there are $952\times 952$ pixels, and each pixel is recorded at 156 channels that cover the wavelengths from 401 nm to 889 nm.  We  use a region of  $95\times 95$ pixels starting from $(252,332)$ pixel in the original image. Three different materials are in this image, they are ``Tree'', ``Rock'', and ``Water'' respectively. We form a 3-order tensor $\mathcal{A}\in \mathbb{R}^{156\times 95\times 95}$, where 156 represents the number of spectral bands, 95 and 95 denote the row and column number of the hypercube, respectively. Moreover, we set $r=3$ here.

\begin{figure*}[!t]
\mbox{\hspace{-8mm}}\includegraphics[width=0.55\textwidth,height=2.7cm]{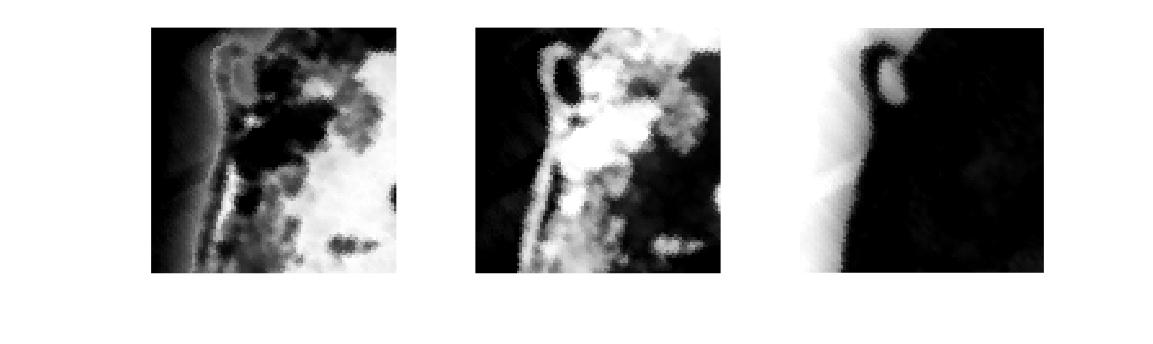}\includegraphics[width=0.55\textwidth,height=2.7cm]{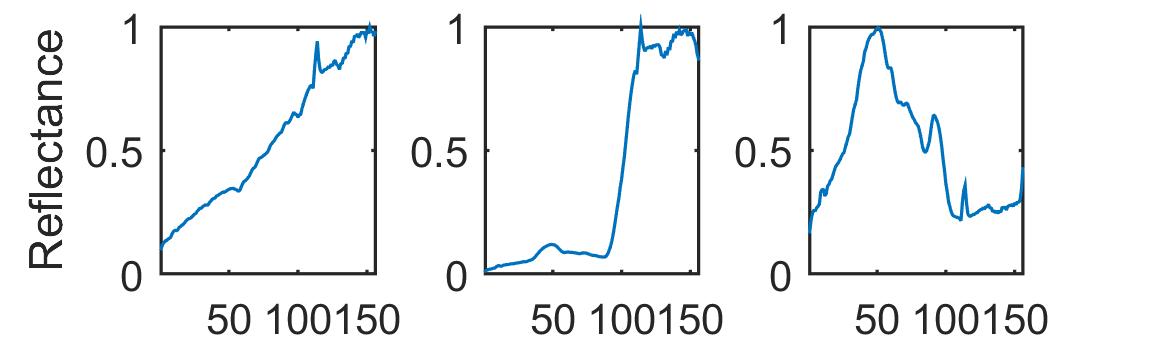} \\
\mbox{\hspace{-8mm}}\includegraphics[width=0.55\textwidth,height=2.7cm]{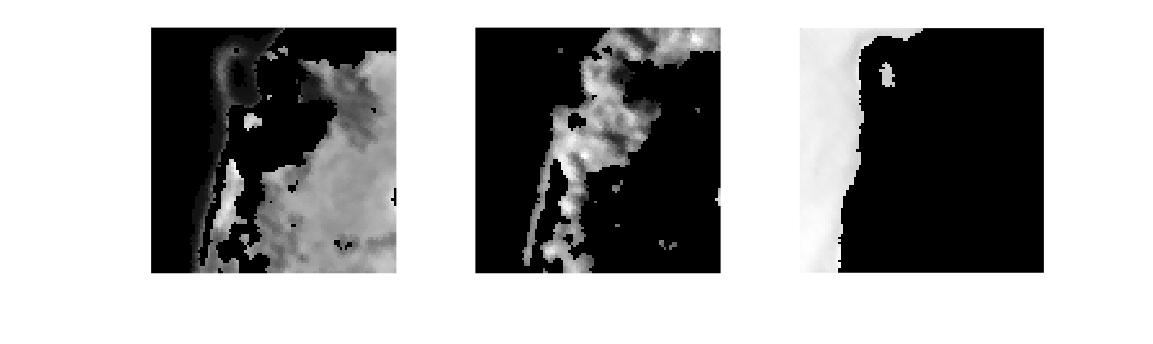}\includegraphics[width=0.55\textwidth,height=2.7cm]{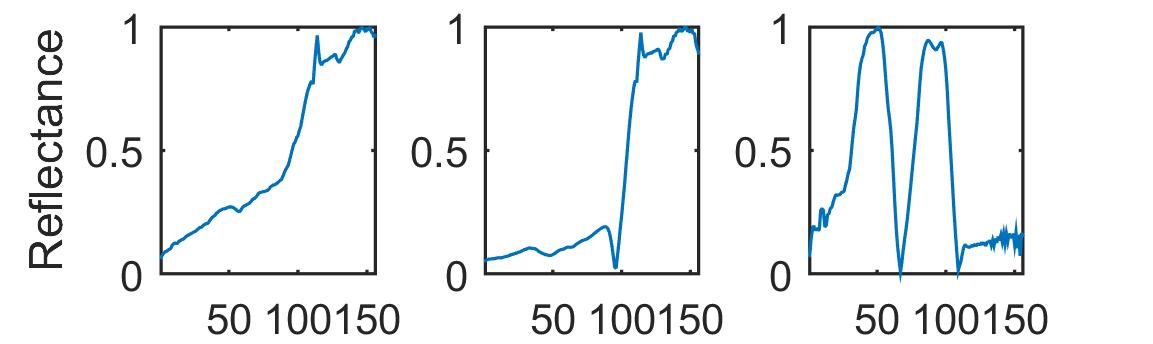}
\\
\mbox{\hspace{-8mm}}\includegraphics[width=0.55\textwidth,height=2.7cm]{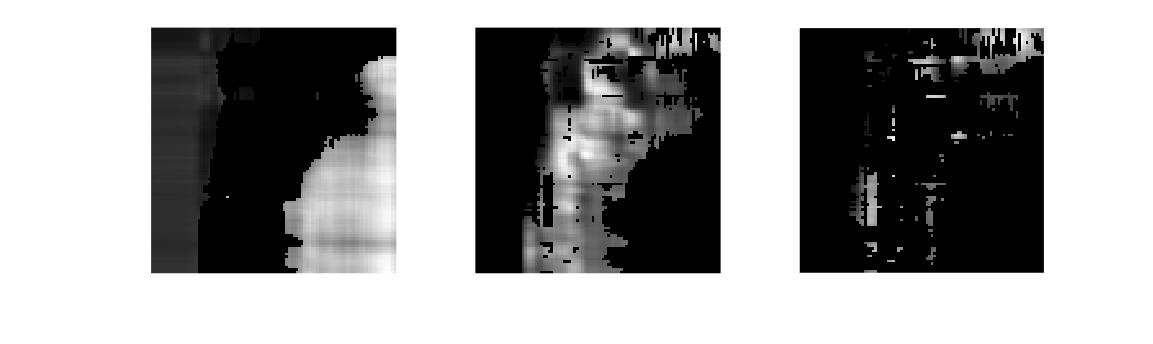}\includegraphics[width=0.55\textwidth,height=2.7cm]{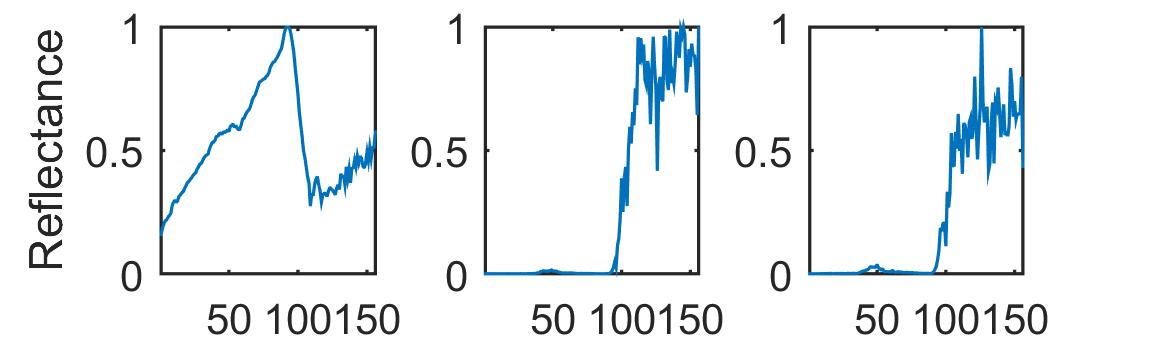} \\
\mbox{\hspace{-8mm}}\includegraphics[width=0.55\textwidth,height=2.7cm]{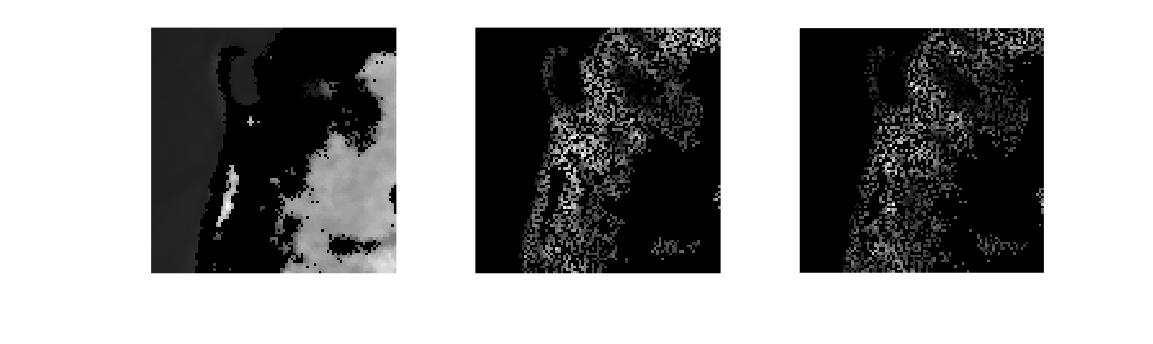}\includegraphics[width=0.55\textwidth,height=2.7cm]{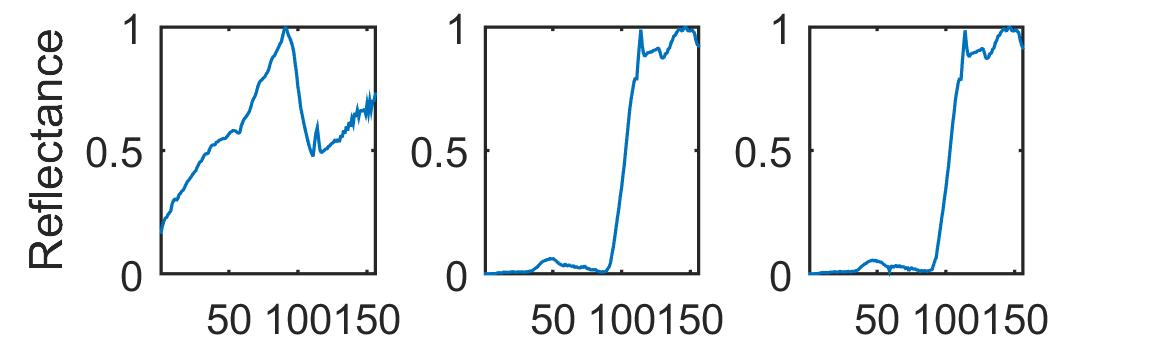} \\
\mbox{\hspace{-8mm}}\includegraphics[width=0.55\textwidth,height=2.7cm]{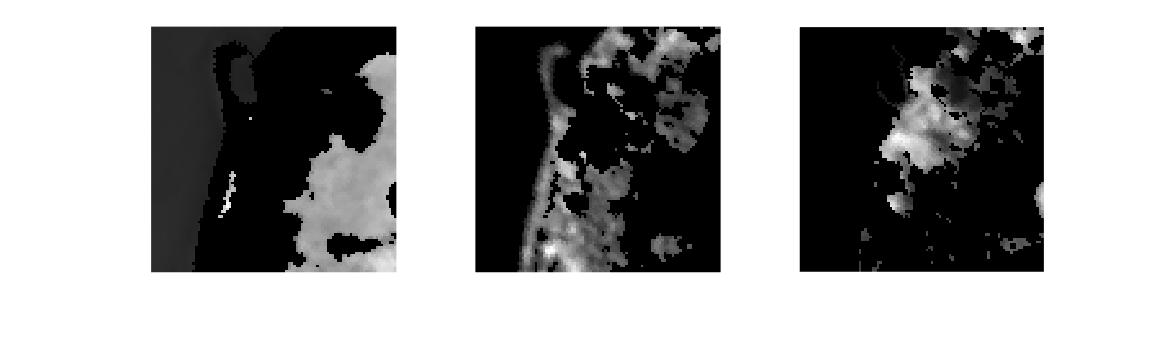}\includegraphics[width=0.55\textwidth,height=2.7cm]{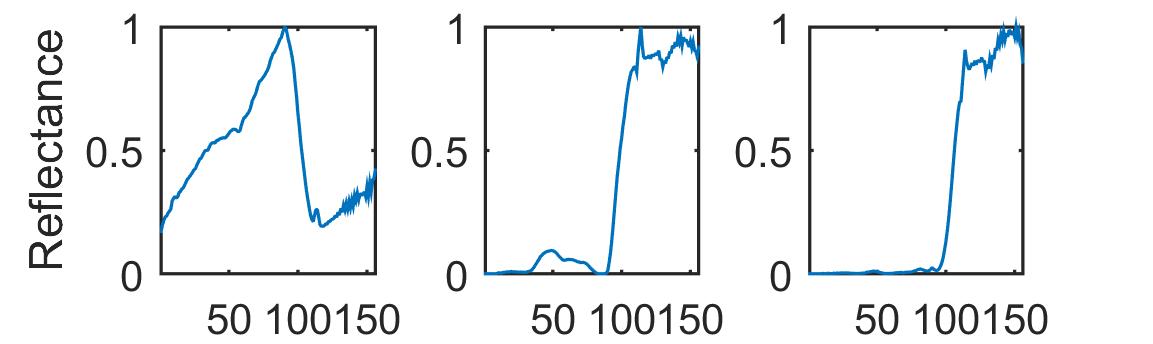} \\
\mbox{\hspace{-8mm}}\includegraphics[width=0.55\textwidth,height=2.7cm]{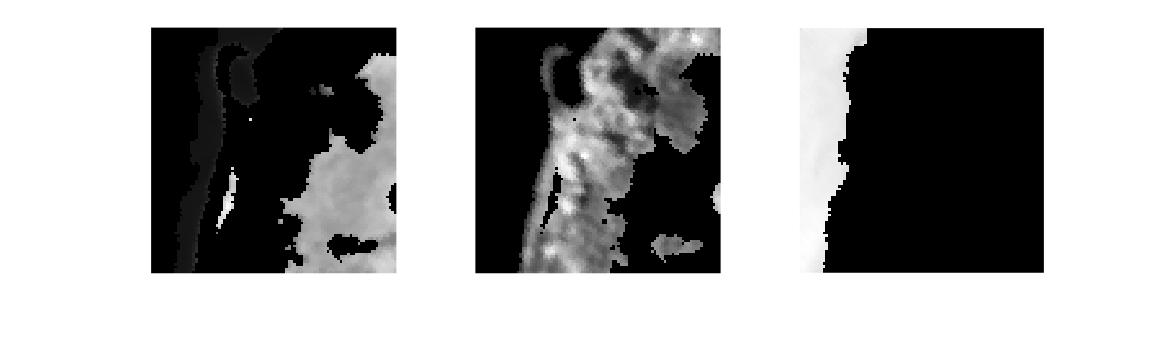}\includegraphics[width=0.55\textwidth,height=2.7cm]{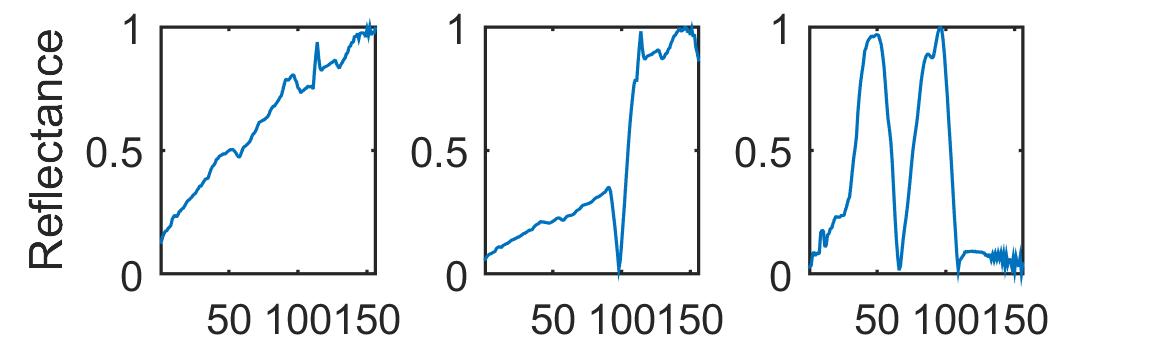} \\
\caption{Left: Rock, Tree, Water; Right: Reflectance of Rock, Tree, Water. From the top to  bottom: groundtruth, ONTD, NTD, NTD1, NMF, SN-ONMF.}\label{FiHp}
\end{figure*}

 In this example, the aim of NTD is finding three factor matrices and one core tensor. We do a minor revision on NTD, i.e., using their method to find the 1st factor matrix size of $156\times 3$ and the core tensor size of $3\times 95\times 95$, we refer it to NTD1 here. Now we use tensor methods (ONTD, NTD, NTD1) on tensor $\mathcal{A}$, then get a $156\times 3$ factor matrix $\mathbf{U}^{(1)}$ and a $3\times 95\times 95$ factor tensor $\mathcal{S}$. The $i$-th feature is obtained by hard clustering\cite{Ding06} base on the first array of tensor $\mathcal{S}$, where $i=1,2,3$. When we use matrix methods (NMF, SN-ONMF), tensor $\mathcal{A}$ should be firstly reshaped to a pixels $\times$ spectral matrix $\mathbf{A}$ whose size is $9025\times 156$. The factor matrices size of $9025\times 3$ and $3\times 156$ can be obtained by the matrix methods. After hard clustering is used on the $9025\times 3$ matrix, we can obtain the $i$-th feature image by reshaping its $i$-th column to a $95\times 95$ matrix. Worth to say, different from \cite{Pan18}, here we impose the orthogonal constraint on spectral-class matrix whose size is $156\times 3$. Note that we don't apply HOOI because the nonnegative value makes it difficult shown in unmixing.

 The groundtruth and the numerical results are displayed in Figure \ref{FiHp}.
 According to the results, our method displays a good clustering performance.
 The tree, rock and water are extracted. But for the other methods,
 they do not perform well because they are not able to separate these materials well.

\begin{table}[htbp]
  \centering
  \caption{The hyperspectral unmixing results of different methods.}\label{10}
  \begin{tabular}{|c|c|c|c|c|c|c|}
  \hline
   ~~ & ONTD  & NTD&NTD1& NMF& SN-ONMF \\
  \hline
 similarity &$\mathbf{0.9083}$&0.5674& 0.5237&0.5142 &0.8960\\
  \hline
   comp. time &$\mathbf{2.62}$s & 82.03s&  16.88s  & 13.87s& 5.47s\\
   \hline
\end{tabular}
\end{table}

To evaluate the result, in this case, we define the following similarity metric:
$$
similarity=\frac{1}{r}\sum^{r}_{i=1}\frac{<\mathbf{h}_i, \mathbf{g}_i>}{\|\mathbf{h}_i\|_2\|\mathbf{g}_i\|_2}
$$
where $\{\mathbf{h}_i\}^{r}_{i=1}$ are the extracted features, $\{\mathbf{g}_i\}^{r}_{i=1}$ are the groundtruth features. It describes the similarity of the groundtruth feature space and the computed feature space. Larger value implies a better result. The clustering similarity and computational time (in seconds) of the above results are list in Table \ref{10}.
Compare to other methods, our model shows the highest similarity and it is the fastest. One can also find that NTD takes more time than the other methods. For matrix method, although SNONMF takes 5.47s, the similarity is high at 0.8960, only less than ONTD.

\section{Conclusion}
In this paper we have studied the orthogonal nonnegative Tucker decomposition problem and developed a structured convex optimization algorithm. The convergence analysis is given. We employ ONTD on the image data sets from the real world applications including face recognition, image representation, hyperspectral unmixing. It shows a good performance.

%


%

 \section*{Appendix}

In the following, we give the proof of Theorem 1. For simplicity, from the algorithm, we give a general model as follows,
\begin{equation*}
\begin{split}
\min F(X,Y,Z,K)=\frac{1}{2}\|A-KA\|^2_F+\theta \|X\|_1 \\
\mbox{s.t.}\quad \left(
              \begin{array}{c}
                I \\
                I \\
                I \\
              \end{array}
            \right)K+\left(
                       \begin{array}{ccc}
                         -I & 0 & 0 \\
                         0 & -I & 0 \\
                         0 & 0 & -I \\
                       \end{array}
                     \right)\left(
                              \begin{array}{c}
                                X \\
                                Z \\
                                M \\
                              \end{array}
                            \right)=0,\\
                            tr(K)=J,~M^T=M,~0\preceq M\preceq I,~Z\geq0.
\end{split}
\end{equation*}
It can be written as

\begin{equation*}
\begin{split}
\min F(K,M,X,Z)=f(K)+h(M)+g(X)+\delta (Z)\\
\mbox{s.t.}\quad  \left(
              \begin{array}{c}
                I \\
                I \\
                I \\
              \end{array}
            \right)K+\left(
                       \begin{array}{ccc}
                         -I & 0 & 0 \\
                         0 & -I & 0 \\
                         0 & 0 & -I \\
                       \end{array}
                     \right)\left(
                              \begin{array}{c}
                                X \\
                                Z \\
                                M \\
                              \end{array}
                            \right)=0,
\end{split}
\end{equation*}
where $f(K)=\frac{1}{2}\|A-KA\|^2_F+f_1(K)$, and
$f_1(K)=\left\{\begin{array}{cl}
0, & \mbox{if} \ tr(K)=J;\\
+\infty, & otherwise.
\end{array}\right.$, $g(X)=\theta \|X\|_1$,
$h(M)=\left\{\begin{array}{cl}
0, & \mbox{if} \ M^T=M,~0\preceq M\preceq I;\\
+\infty. & otherwise.
\end{array}\right.$,
$
\delta(Z)=\left\{\begin{array}{cl}
0, & \mbox{if} \ Z\in \mathbb{R}_+;\\
+\infty, & otherwise.
\end{array}\right.
$.
The general model can be simply written as
\begin{equation}\label{conve1}
  \min F=f(K)+h(M)+g(X)+\delta(Z)\quad \mbox{s.t.}\quad K=X,~K=Z,~K=M.
\end{equation}
The algorithm for the general model is in the following:
\begin{equation*}
\begin{split}
&K^i=\arg\min \{f(K)+\frac{\lambda_1}{2}\|K-M^{i-1}+b^{i-1}_1\|^2_F+\frac{\lambda_2}{2}\|K-X^{i-1}+b^{i-1}_2\|^2_F+\frac{\lambda_3}{2}\|K-Z^{i-1}+b^{i-1}_3\|^2_F\};\\
&M^i=\arg\min\{h(M)+\frac{\lambda_1}{2}\|K^{i}-M+b^{i-1}_1\|^2_F \};\\
&X^i=\arg\min\{g(X)+\frac{\lambda_2}{2}\|K^{i}-X+b^{i-1}_2\|^2_F \};\\
&Z^i=\arg\min\{\delta(Z)+\frac{\lambda_3}{2}\|K^{i}-Z+b^{i-1}_3\|^2_F  \};\\
&b^i_1=b^{i-1}_1+K^{i}-M^{i};\\
&b^i_2=b^{i-1}_2+K^{i}-X^{i};\\
&b^i_3=b^{i-1}_3+K^{i}-Z^{i}.\\
\end{split}
\end{equation*}
The lagrangian can be written as
\begin{equation}\label{conv_lang}
\begin{split}
L=&f(K)+h(M)+g(X)+\delta(Z)+\frac{\lambda_1}{2}\|K-M\|^2_F+\lambda_1\langle K-M, b_1\rangle+\frac{\lambda_2}{2}\|K-X\|^2_F\\
&+\lambda_2\langle K-X, b_2\rangle+\frac{\lambda_3}{2}\|K-Z\|^2_F+\lambda_3\langle K-Z, b_3\rangle.
\end{split}
\end{equation}
\begin{Definition}
$(K^*,M^*,X^*,Z^*,b^*_1,b^*_2,b^*_3)$ is a saddle point if
\begin{equation}\label{conv_saddle}
L(K^*,M^*,X^*,Z^*,b_1,b_2,b_3)\leq L(K^*,M^*,X^*,Z^*,b^*_1,b^*_2,b^*_3)\leq L(K,M,X,Z,b^*_1,b^*_2,b^*_3)
\end{equation}
for any $(K,M,X,Z,b_1,b_2,b_3)$.
\end{Definition}
\begin{lemma}\label{lemm-saddle}
 $(K^*,M^*,X^*,Z^*)$ is a solution of problem (\ref{conve1}) if and only if there exist $b^*_1$, $b^*_2$, $b^*_3$ such that $(K^*,M^*,X^*,Z^*,b^*_1,b^*_2,b^*_3)$ is a saddle point.
\end{lemma}
\begin{proof}
$(K^*,M^*,X^*,Z^*,b^*_1,b^*_2,b^*_3)$ is a saddle point. From
$$
L(K^*,M^*,X^*,Z^*,b_1,b_2,b_3)\leq L(K^*,M^*,X^*,Z^*,b^*_1,b^*_2,b^*_3),
$$
we get, $\forall b_1, b_2, b_3$,
$$
\lambda_1\langle K^*-M^*, b_1\rangle+\lambda_2\langle K^*-X^*, b_2\rangle+\lambda_3\langle K^*-Z^*, b_3\rangle\leq \lambda_1\langle K^*-M^*, b^*_1\rangle+\lambda_2\langle K^*-X^*, b^*_2\rangle+\lambda_3\langle K^*-Z^*, b^*_3\rangle.
$$
Let $b_1=b^*_1$, $b_2=b^*_2$, $b_3=b^*_3\pm \triangle b_3$, we get $K^*=Z^*$. Similarly, $K^*=M^*$, $K^*=X^*$. Also  $\forall  (K,M,X,Z)$,
$$
L(K^*,M^*,X^*,Z^*,b^*_1,b^*_2,b^*_3)\leq L(K,M,X,Z,b^*_1,b^*_2,b^*_3),
$$
let $K=M=X=Z$,  we have
$$f(K^*)+h(M^*)+g(X^*)+\delta (Z^*)\leq f(K)+h(M)+g(X)+\delta (Z),$$
i.e., $(K^*,M^*,X^*,Z^*)$ is minimizer of $F(K,M,X,Z)$.

If $(K^*,M^*,X^*,Z^*)$ is solution of (\ref{conve1}), i.e., $K^*=M^*=X^*=Z^*$, then the left inequality of (\ref{conv_saddle}) established. Moreover, $\exists b^*_1, b^*_2, b^*_3$ such that
$$
-\lambda_1b^*_1-\lambda_2b^*_2-\lambda_3b^*_3\in \partial f(K^*),~~\lambda_1b^*_1\in \partial h(M^*),~~\lambda_2b^*_2\in \partial g(X^*),~~\lambda_3b^*_3\in \partial \delta(Z^*).
$$
For $\forall M, X, Z,$
$$
f(K)-f(K^*)\geq \nabla f(K^*)(K-K^*),~~h(M)-h(M^*)\geq \nabla h(M^*)(M-M^*),
$$
$$
g(X)-g(X^*)\geq \nabla g(X^*)(X-X^*),~~ \delta(Z)-\delta(Z^*)\geq \nabla \delta(Z^*)(Z-Z^*),
$$
i.e.,
\begin{eqnarray}\label{conve2}
\notag
f(K^*)&\leq& f(K)+\langle \lambda_1b^*_1+\lambda_2b^*_2+\lambda_3b^*_3,K-K^*\rangle,\\
h(M^*)&\leq&h(M)-\langle \lambda_1b^*_1, M-M^*\rangle,\\
\notag
g(X^*)&\leq& g(X)-\langle \lambda_2b^*_2, X-X^*\rangle,\\
\notag
\delta(Z^*)&\leq& \delta(Z)-\langle \lambda_3b^*_3, Z-Z^*\rangle,
\end{eqnarray}
the summation of above inequality (\ref{conve2}) yields the right inequality of (\ref{conv_saddle}).
\end{proof}

\begin{theorem}
$\{(K^i, M^i,X^i,Z^i)\}$ generated by algorithm from any starting point converges to a minimum of (\ref{conve1}).
\end{theorem}
\begin{proof}
Let $(K^*, M^*, X^*,Z^*)$ be an optimal solution, from algorithm,
\begin{eqnarray*}
K^*&=&\arg\min L(K,M^*, X^*, Z^*,b^*_1,b^*_2,b^*_3);\\
M^*&=&\arg\min L(K^*,M, X^*, Z^*,b^*_1,b^*_2,b^*_3);\\
X^*&=&\arg\min L(K^*,M^*, X, Z^*,b^*_1,b^*_2,b^*_3);\\
Z^*&=&\arg\min L(K^*,M^*, X^*, Z,b^*_1,b^*_2,b^*_3);
\end{eqnarray*}
for any $K,M,X,Z\in \mathbb{R}^{m\times m}$. From (\ref{conve2}), we have,
\begin{eqnarray}\label{conve3}
\notag
 &f(K)-f(K^*)+\lambda_1\langle K-K^*,b^*_1\rangle+\lambda_2\langle K-K^*,b^*_2\rangle+\lambda_3\langle K-K^*,b^*_3\rangle\\
\notag
&=f(K)-f(K^*)+\lambda_1\langle K-K^*,K^*-M^*+b^*_1\rangle\\
\notag
& +\lambda_2\langle K-K^*,K^*-X^*+b^*_2\rangle+\lambda_3\langle K-K^*,K^*-Z^*+b^*_3\rangle\geq 0,\\
&h(M)-h(M^*)+\lambda_1\langle M^*-K^*-b^*_1, M-M^*\rangle \geq  0,\\
\notag
&g(X)-g(X^*)+\lambda_2\langle X^*-K^*-b^*_2, X-X^*\rangle \geq  0,\\
\notag
&\delta(Z)-\delta(Z^*)+\lambda_3\langle Z^*-K^*-b^*_3, Z-Z^*\rangle \geq  0.\\
\notag
\end{eqnarray}
Because of  the algorithm, $(K^i,M^i,X^i,Z^i)$ for any $K,M,X,Z$, we have
$$
f(K)-f(K^i)\geq \nabla f^T(K^i)(K-K^i),
$$
where $$\nabla f^T(K^i)=-\lambda_1( K^{i}-M^{i-1}+b^{i-1}_1)-\lambda_2(K^i-X^{i-1}+b^{i-1}_2)-\lambda_3(K^i-Z^{i-1}+b^{i-1}_3).$$
So,
\begin{eqnarray}\label{conve4-1}
f(K)&-&f(K^i)+\lambda_1\langle K^i-M^{i-1}+b^{i-1}_1, K-K^i\rangle\\
\notag
&+&\lambda_2\langle K^i-X^{i-1}+b^{i-1}_2, K-K^i\rangle+\lambda_3\langle K^i-Z^{i-1}+b^{i-1}_3, K-K^i\rangle\geq 0.
\end{eqnarray}
Similar, we get
\begin{eqnarray}\label{conve4-2}
\notag
h(M)-h(M^i)+\lambda_1\langle M^i-K^i-b^{i-1}_1,M-M^i\rangle&\geq& 0,\\
\notag
g(X)-g(X^i)+\lambda_2\langle X^i-K^i-b^{i-1}_2,X-X^i\rangle&\geq &0,\\
\delta(Z)-\delta(Z^i)+\lambda_3\langle Z^i-K^i-b^{i-1}_3,Z-Z^i\rangle&\geq& 0.\\
\notag
\end{eqnarray}
Let $K=K^i$ in (\ref{conve3}) and $K=K^*$ in (\ref{conve4-1}), then
$$
f(K^i)-f(K^*)+\lambda_1\langle K^i-K^*, K^*-M^*+b^*_1\rangle + \lambda_2\langle K^i-K^*, K^*-X^*+b^*_2\rangle+ \lambda_3\langle K^i-K^*, K^*-Z^*+b^*_3\rangle\geq 0,
$$
\begin{equation*}
 \begin{split}
&f(K^*)-f(K^i)+\lambda_1\langle K^*-K^i, K^i-M^{i-1}+b^{i-1}_1\rangle + \lambda_2\langle K^*-K^i, K^i-X^{i-1}+b^{i-1}_2\rangle\\
&+ \lambda_3\langle K^*-K^i, K^i-Z^{i-1}+b^{i-1}_3\rangle\geq 0.
\end{split}
\end{equation*}
The summation of above inequality is
\begin{eqnarray*}
  &\lambda_1\langle K^i-K^*, K^*-M^*+b^*_1-K^i+M^{i-1}-b^{i-1}_1\rangle
  + \lambda_2\langle K^i-K^*, K^*-X^*+b^*_2-K^i+X^{i-1}-b^{i-1}_2\rangle\\
 &+  \lambda_3\langle K^i-K^*, K^*-Z^*+b^*_3-K^i+Z^{i-1}-b^{i-1}_3\rangle\geq 0.
\end{eqnarray*}
Similarly, we get
\begin{eqnarray*}
\lambda_1 \langle M^i-M^*, M^*-K^*-b^*_1-M^i+K^{i}+b^{i-1}_1\rangle \geq 0,\\
\lambda_2 \langle X^i-X^*, X^*-K^*-b^*_2-X^i+K^{i}+b^{i-1}_2\rangle \geq 0,\\
\lambda_3 \langle Z^i-Z^*, Z^*-K^*-b^*_3-Z^i+K^{i}+b^{i-1}_3\rangle \geq 0,\\
\end{eqnarray*}
Let $\triangle K^i =K^i-K^*$, $\triangle M^i =M^i-M^*$, $\triangle X^i =X^i-X^*$, $\triangle Z^i =Z^i-Z^*$, $\triangle b^i_1 =b^i_1-b^*_1$, $\triangle b^i_2 =b^i_2-b^*_2$, $\triangle b^i_3 =b^i_3-b^*_3$, so

\begin{eqnarray}\label{conve5}
\notag
 &\lambda_1 \langle\triangle K^i, \triangle M^{i-1}-\triangle K^i-\triangle b^{i-1}_1\rangle+ \lambda_2 \langle\triangle K^i, \triangle X^{i-1}-\triangle K^i-\triangle b^{i-1}_2\rangle \\
 \notag
&+\lambda_3 \langle\triangle K^i, \triangle Z^{i-1}-\triangle K^i-\triangle b^{i-1}_3\rangle \geq 0,\\
\notag
 &\lambda_1\langle \triangle M^i, \triangle K^i-\triangle M^i+\triangle b^{i-1}_1\rangle \geq 0,\\
 &\lambda_2\langle \triangle X^i, \triangle K^i-\triangle X^i+\triangle b^{i-1}_2\rangle \geq 0,\\
\notag
&\lambda_3\langle \triangle Z^i, \triangle K^i-\triangle Z^i+\triangle b^{i-1}_3\rangle \geq 0.\\
\notag
\end{eqnarray}
Summation of the above inequalities, we get that
\begin{eqnarray*}
\lambda_1( \langle \triangle M^i-\triangle K^i, \triangle b^{i-1}_1\rangle+\langle \triangle M^i, \triangle K^{i}-\triangle M^{i}\rangle+\langle \triangle K^i, \triangle M^{i-1}-\triangle K^{i}\rangle)\\
+\lambda_2( \langle \triangle X^i-\triangle K^i, \triangle b^{i-1}_2\rangle+\langle \triangle X^i, \triangle K^{i}-\triangle X^{i}\rangle+\langle \triangle K^i, \triangle X^{i-1}-\triangle K^{i}\rangle)\\
+\lambda_3( \langle \triangle Z^i-\triangle K^i, \triangle b^{i-1}_3\rangle+\langle \triangle Z^i, \triangle K^{i}-\triangle Z^{i}\rangle+\langle \triangle K^i, \triangle Z^{i-1}-\triangle K^{i}\rangle)\geq 0,
\end{eqnarray*}
i.e.,
\begin{eqnarray*}
\lambda_1\langle \triangle M^i-\triangle K^i, \triangle b^{i-1}_1\rangle-\lambda_1\|\triangle K^i-\triangle M^i\|^2-\lambda_1 \langle \triangle M^i-\triangle M^{i-1}, \triangle K^{i}\rangle\\
 + \lambda_2\langle \triangle X^i-\triangle K^i, \triangle b^{i-1}_2\rangle-\lambda_2\|\triangle K^i-\triangle X^i\|^2-\lambda_2 \langle \triangle X^i-\triangle X^{i-1}, \triangle K^{i}\rangle\\
 + \lambda_3\langle \triangle Z^i-\triangle K^i, \triangle b^{i-1}_3\rangle-\lambda_3\|\triangle K^i-\triangle Z^i\|^2-\lambda_3 \langle \triangle Z^i-\triangle Z^{i-1}, \triangle K^{i}\rangle\geq 0.
\end{eqnarray*}

On the other hand,
\begin{eqnarray*}
\triangle b^i_1&=&\triangle b^{i-1}_1+\triangle K^i-\triangle M^i,\\
\triangle b^i_2&=&\triangle b^{i-1}_2+\triangle K^i-\triangle X^i,\\
\triangle b^i_3&=&\triangle b^{i-1}_3+\triangle K^i-\triangle Z^i.\\
\end{eqnarray*}
We get that,
\begin{eqnarray}\label{conve-lambda}
&\lambda_1 (\|\triangle b^{i-1}_1\|^2-\|\triangle b^i_1\|^2)+ \lambda_2(\|\triangle b^{i-1}_2\|^2-\|\triangle b^i_2\|^2)+\lambda_3 (\|\triangle b^{i-1}_3\|^2-\|\triangle b^i_3\|^2)\\
\notag
&=\lambda_1(-2\langle  \triangle K^i-\triangle M^i, \triangle b^{i-1}_1 \rangle-\|\triangle K^i-\triangle M^i\|^2)+\lambda_2(-2\langle \triangle K^i-\triangle X^i, \triangle b^{i-1}_2 \rangle\\
\notag
&-\|\triangle K^i-\triangle X^i\|^2)+\lambda_3(-2\langle  \triangle K^i-\triangle Z^i, \triangle b^{i-1}_3 \rangle-\|\triangle K^i-\triangle Z^i\|^2)\\
\notag
&\geq 2 \lambda_1 (\|\triangle K^i-\triangle M^i\|^2+\langle \triangle M^i-\triangle M^{i-1}, \triangle K^i\rangle)+2 \lambda_2 (\|\triangle K^i-\triangle X^i\|^2\\
\notag
&+\langle \triangle X^i-\triangle X^{i-1}, \triangle K^i\rangle)+2 \lambda_3 (\|\triangle K^i-\triangle Z^i\|^2+\langle \triangle Z^i-\triangle Z^{i-1}, \triangle K^i\rangle)\\\notag
&-\lambda_1\|\triangle K^i-\triangle M^i\|^2-\lambda_2\|\triangle K^i-\triangle X^i\|^2-\lambda_3\|\triangle K^i-\triangle Z^i\|^2\\
\notag
&=\lambda_1\|\triangle K^i-\triangle M^i\|^2+\lambda_2\|\triangle K^i-\triangle X^i\|^2+\lambda_3\|\triangle K^i-\triangle Z^i\|^2\\
\notag
&+2 \lambda_1 \langle\triangle M^i-\triangle M^{i-1},\triangle K^i\rangle+2 \lambda_2 \langle\triangle X^i-\triangle X^{i-1},\triangle K^i\rangle+2 \lambda_3 \langle\triangle Z^i-\triangle Z^{i-1},\triangle K^i\rangle.
\end{eqnarray}
Note that,
\begin{eqnarray*}
M^{i-1}&=&\arg\min_M{h(M)+\frac{\lambda_1}{2}\|K^{i-1}-M+b^{i-2}_1\|^2_F},\\
X^{i-1}&=&\arg\min_X{g(X)+\frac{\lambda_2}{2}\|K^{i-1}-X+b^{i-2}_2\|^2_F},\\
Z^{i-1}&=&\arg\min_Z{\delta(Z)+\frac{\lambda_3}{2}\|K^{i-1}-Z+b^{i-2}_3\|^2_F},
\end{eqnarray*}
so for any $M, X, Z\in \mathbb{R}^{m\times m}$,
$$
h(M)-h(M^{i-1})\geq (\nabla h(M^{i-1}))^T(M-M^{i-1}),
$$
where $\nabla h(M^{i-1})=-\lambda_1(-K^{i-1}+M^{i-1}-b^{i-2}_1)$. Therefore,
\begin{eqnarray}\label{conve6-1}
h(M)-h(M^{i-1})+\lambda_1 \langle M^{i-1}-K^{i-1}-b^{i-2}_1, M-M^{i-1}\rangle\geq 0,
\end{eqnarray}
similar,
\begin{eqnarray}\label{conve6-2}
g(X)-g(X^{i-1})+\lambda_2 \langle X^{i-1}-K^{i-1}-b^{i-2}_2, X-X^{i-1}\rangle &\geq& 0,
\end{eqnarray}
\begin{eqnarray}\label{conve6-3}
\delta(Z)-\delta(Z^{i-1})+\lambda_3 \langle Z^{i-1}-K^{i-1}-b^{i-2}_3, Z-Z^{i-1}\rangle &\geq& 0.
\end{eqnarray}
Let $M=M^i, X=X^i, Z=Z^i$ in (\ref{conve6-1}, \ref{conve6-2}, \ref{conve6-3}), then,
\begin{eqnarray}\label{conve7}
\notag
h(M^i)-h(M^{i-1})+\lambda_1\langle M^{i-1}-K^{i-1}-b^{i-2}_1, M^i-M^{i-1}\rangle &\geq & 0 ,\\
g(X^i)-g(X^{i-1})+\lambda_2\langle X^{i-1}-K^{i-1}-b^{i-2}_2, X^i-X^{i-1}\rangle &\geq & 0 ,\\
\notag
\delta(Z^i)-\delta(Z^{i-1})+\lambda_3\langle Z^{i-1}-K^{i-1}-b^{i-2}_3, Z^i-Z^{i-1}\rangle &\geq & 0.\\
\notag
\end{eqnarray}
Let $M=M^{i-1}, X=X^{i-1}, Z=Z^{i-1}$ in (\ref{conve4-2}), we have,
\begin{eqnarray}\label{conve8}
\notag
h(M^{i-1})-h(M^i)+\lambda_1\langle M^{i}-K^{i}-b^{i-1}_1, M^{i-1}-M^{i}\rangle &\geq & 0 ,\\
g(X^{i-1})-g(X^i)+\lambda_2\langle X^{i}-K^{i}-b^{i-1}_2, X^{i-1}-X^{i}\rangle &\geq & 0 ,\\
\notag
\delta(Z^{i-1})-\delta(Z^{i})+\lambda_3\langle Z^{i}-K^{i}-b^{i-1}_3, Z^{i-1}-Z^{i}\rangle &\geq & 0.\\
\notag
\end{eqnarray}
Summation of (\ref{conve7}) and (\ref{conve8}),
\begin{eqnarray}\label{conve9}
\notag
 \langle b^{i-1}_1-b^{i-2}_1, M^{i}-M^{i-1}\rangle + \langle M^{i-1}-K^{i-1}-M^{i}+K^{i}, M^i-M^{i-1}\rangle&\geq & 0 ,\\
\langle b^{i-1}_2-b^{i-2}_2, X^{i}-X^{i-1}\rangle + \langle X^{i-1}-K^{i-1}-X^{i}+K^{i}, X^i-X^{i-1}\rangle&\geq & 0  ,\\
\notag
\langle b^{i-1}_3-b^{i-2}_3, Z^{i}-Z^{i-1}\rangle + \langle Z^{i-1}-K^{i-1}-Z^{i}+K^{i}, Z^i-Z^{i-1}\rangle&\geq & 0 .
\notag
\end{eqnarray}
Also we know,
\begin{eqnarray*}
&\triangle K^i-\triangle K^{i-1}=K^i-K^{i-1},~~\triangle M^i-\triangle M^{i-1}=M^i-M^{i-1},~~\triangle X^i-\triangle X^{i-1}=X^i-X^{i-1},\\
&\triangle Z^i-\triangle Z^{i-1}=Z^i-Z^{i-1},~~b^{i-1}_1-b^{i-2}_1=\triangle K^{i-1}-\triangle M^{i-1},~~b^{i-1}_2-b^{i-2}_2=\triangle K^{i-1}-\triangle X^{i-1},\\
&b^{i-1}_3-b^{i-2}_3=\triangle K^{i-1}-\triangle Z^{i-1}.
\end{eqnarray*}
(\ref{conve9}) can be written as
\begin{eqnarray*}
\langle \triangle K^{i-1}-\triangle M^{i-1}, \triangle M^{i}-\triangle M^{i-1}\rangle+\langle \triangle K^i-\triangle K^{i-1}, \triangle M^i-\triangle M^{i-1}\rangle &\geq & \|\triangle M^i-\triangle M^{i-1}\|^2_F, \\
\langle \triangle K^{i-1}-\triangle X^{i-1}, \triangle X^{i}-\triangle X^{i-1}\rangle+\langle \triangle K^i-\triangle K^{i-1}, \triangle X^i-\triangle X^{i-1}\rangle &\geq & \|\triangle X^i-\triangle X^{i-1}\|^2_F,\\
\langle \triangle K^{i-1}-\triangle Z^{i-1}, \triangle Z^{i}-\triangle Z^{i-1}\rangle+\langle \triangle K^i-\triangle K^{i-1}, \triangle Z^i-\triangle Z^{i-1}\rangle &\geq & \|\triangle Z^i-\triangle Z^{i-1}\|^2_F. \\
\end{eqnarray*}
Because
\begin{eqnarray*}
&\langle \triangle M^{i}-\triangle M^{i-1},\triangle K^i \rangle =\langle \triangle M^i-\triangle M^{i-1}, \triangle K^i-\triangle K^{i-1}\rangle
\\&+\langle \triangle M^i-\triangle M^{i-1}, \triangle K^{i-1}-\triangle M^{i-1}\rangle
+\langle \triangle M^i-\triangle M^{i-1}, \triangle M^{i-1}\rangle,\\
&\langle \triangle X^{i}-\triangle X^{i-1},\triangle K^i \rangle =  \langle \triangle X^i-\triangle X^{i-1}, \triangle K^i-\triangle K^{i-1}\rangle
\\ &+\langle\triangle X^i-\triangle X^{i-1}, \triangle K^{i-1}-\triangle X^{i-1}\rangle
 +\langle \triangle X^i-\triangle X^{i-1}, \triangle X^{i-1}\rangle,\\
&\langle \triangle Z^{i}-\triangle Z^{i-1},\triangle K^i \rangle =\langle \triangle Z^i-\triangle Z^{i-1}, \triangle K^i-\triangle K^{i-1}\rangle
\\&+ \langle \triangle Z^i-\triangle Z^{i-1}, \triangle K^{i-1}-\triangle Z^{i-1}\rangle
 +\langle \triangle Z^i-\triangle Z^{i-1}, \triangle Z^{i-1}\rangle,\\
\end{eqnarray*}
then
\begin{eqnarray*}
\langle \triangle M^i-\triangle M^{i-1},\triangle K^i\rangle &\geq& \|\triangle M^i-\triangle M^{i-1}\|^2_F+\langle \triangle M^{i-1}, \triangle M^{i}-\triangle M^{i-1}\rangle,\\
\langle \triangle X^i-\triangle X^{i-1},\triangle K^i\rangle &\geq& \|\triangle X^i-\triangle X^{i-1}\|^2_F+\langle \triangle X^{i-1}, \triangle X^{i}-\triangle X^{i-1}\rangle,\\
\langle \triangle Z^i-\triangle Z^{i-1},\triangle K^i\rangle &\geq& \|\triangle Z^i-\triangle Z^{i-1}\|^2_F+\langle \triangle Z^{i-1}, \triangle Z^{i}-\triangle Z^{i-1}\rangle.\\
\end{eqnarray*}
So, (\ref{conve-lambda}) can be written as
\begin{eqnarray*}
&\lambda_1 (\|\triangle b^{i-1}_1\|^2_F-\|\triangle b^{i}_1\|^2_F)+\lambda_2 (\|\triangle b^{i-1}_2\|^2_F-\|\triangle b^{i}_2\|^2_F)+\lambda_3 (\|\triangle b^{i-1}_3\|^2_F-\|\triangle b^{i}_3\|^2_F)\\
&\geq \lambda_1 \|\triangle K^i-\triangle M^i\|^2_F+\lambda_2 \|\triangle K^i-\triangle X^i\|^2_F+\lambda_3 \|\triangle K^i-\triangle Z^i\|^2_F\\
& +2\lambda_1\|\triangle M^i-\triangle M^{i-1}\|^2_F+2\lambda_1\langle\triangle M^{i-1},\triangle M^i-\triangle M^{i-1}\rangle\\
&+2\lambda_2\|\triangle X^i-\triangle X^{i-1}\|^2_F+2\lambda_2\langle\triangle X^{i-1},\triangle X^i-\triangle X^{i-1}\rangle\\
&+2\lambda_3\|\triangle Z^i-\triangle Z^{i-1}\|^2_F+2\lambda_3\langle\triangle Z^{i-1},\triangle Z^i-\triangle Z^{i-1}\rangle\\
&=\lambda_1 \|\triangle K^i-\triangle M^i\|^2_F+\lambda_1 \|\triangle M^{i}-\triangle M^{i-1}\|^2_F+\lambda_1(\|\triangle M^i\|^2_F-\|\triangle M^{i-1}\|^2_F)\\
&+\lambda_2 \|\triangle K^i-\triangle X^i\|^2_F+\lambda_2 \|\triangle X^{i}-\triangle X^{i-1}\|^2_F+\lambda_2(\|\triangle X^i\|^2_F-\|\triangle X^{i-1}\|^2_F)\\
&+\lambda_3 \|\triangle K^i-\triangle Z^i\|^2_F+\lambda_3 \|\triangle Z^{i}-\triangle Z^{i-1}\|^2_F+\lambda_3(\|\triangle Z^i\|^2_F-\|\triangle Z^{i-1}\|^2_F).
\end{eqnarray*}
Therefore,
\begin{eqnarray*}
&(\lambda_1 \|\triangle b^{i-1}_1\|^2_F+\lambda_2\|\triangle b^{i-1}_2\|^2_F+\|\triangle b^{i-1}_3\|^2_F+\lambda_1 \|\triangle M^{i-1}\|^2_F+\lambda_2 \|\triangle X^{i-1}\|^2_F+\lambda_3 \|\triangle Z^{i-1}\|^2_F)\\
&-(\lambda_1 \|\triangle b^{i}_1\|^2_F+\lambda_2\|\triangle b^{i}_2\|^2_F+\|\triangle b^{i}_3\|^2_F+\lambda_1 \|\triangle M^{i}\|^2_F+\lambda_2 \|\triangle X^{i}\|^2_F+\lambda_3 \|\triangle Z^{i}\|^2_F)\\
&\geq \lambda_1(\|\triangle K^i-\triangle M^i\|^2_F+\|\triangle M^i-\triangle M^{i-1}\|^2_F)+\lambda_2(\|\triangle K^i-\triangle X^i\|^2_F+\|\triangle X^i-\triangle X^{i-1}\|^2_F)\\
&+\lambda_3(\|\triangle K^i-\triangle Z^i\|^2_F+\|\triangle Z^i-\triangle Z^{i-1}\|^2_F)\geq 0.
\end{eqnarray*}
The sequence $\{ \lambda_1 \|\triangle b^i_1\|^2_F+\lambda_2 \|\triangle b^i_2\|^2_F+\lambda_3 \|\triangle b^i_3\|^2_F+\lambda_1 \|\triangle M^i\|^2_F+\lambda_2 \|\triangle X^i\|^2_F+\lambda_3 \|\triangle Z^i\|^2_F\}$ is non-increasing and convergent.

Also $\{K^i\}, \{M^i\}, \{X^i\}, \{Z^i\}, \{b^i_1\}, \{b^i_2\}, \{b^i_3\} $ are bounded and they have limit points. We  get that $\lim_{i\rightarrow \infty} \|K^i-M^i\|^2_F=0$, $\lim_{i\rightarrow \infty} \|K^i-X^i\|^2_F$, $\lim_{i\rightarrow \infty} \|K^i-Z^i\|^2_F=0$.

Therefore, let $(\tilde{K}, \tilde{M}, \tilde{X}, \tilde{Z}, \tilde{b}_1, \tilde{b}_2, \tilde{b}_3)$ be limit point, i.e., there exists subsequence such that
$$
\lim_{j\rightarrow \infty} (K^{i_j},M^{i_j},X^{i_j},Z^{i_j},b^{i_j}_1, b^{i_j}_2, b^{i_j}_3)=(\tilde{K}, \tilde{M}, \tilde{X}, \tilde{Z}, \tilde{b}_1, \tilde{b}_2, \tilde{b}_3).
$$
In the following, we prove that $(\tilde{K},\tilde{M},\tilde{X},\tilde{Z})$ is minimum, i.e.,
$$
f(\tilde{K})+h(\tilde{M})+g(\tilde{X})+\delta(\tilde{Z})=f(K^*)+h(M^*)+g(X^*)+\delta(Z^*).
$$
Note that $(K^*,M^*,X^*,Z^*)$ is a saddle point, from lemma (\ref{lemm-saddle}), $K^*=M^*=X^*=Z^*$,
\begin{eqnarray*}
&f(K^*)+h(M^*)+g(X^*)+\delta(Z^*)\leq f(K^{i_j})+h(M^{i_j})+g(X^{i_j})+\delta (Z^{i_j})\\
&+\frac{\lambda_1}{2}\|K^{i_j}-M^{i_j}\|^2_F+\frac{\lambda_2}{2}\|K^{i_j}-X^{i_j}\|^2_F+\frac{\lambda_3}{2}\|K^{i_j}-Z^{i_j}\|^2_F\\
&+\lambda_1 \langle K^{i_j}-M^{i_j}, b^*_1\rangle+\lambda_2 \langle K^{i_j}-X^{i_j}, b^*_2\rangle+\lambda_3 \langle K^{i_j}-Z^{i_j}, b^*_3\rangle,
\end{eqnarray*}
$j\rightarrow \infty$, we get that
\begin{equation}\label{conve10-1}
f(K^*)+h(M^*)+g(X^*)+\delta (Z^*)\leq f(\tilde{K})+h(\tilde{M})+g(\tilde{X})+\delta(\tilde{Z}).
\end{equation}
On the other hand, let $K=K^*, M=M^*, X=X^*, Z=Z^*$ in equation (\ref{conve4-1}) and (\ref{conve4-2}), we get
\begin{eqnarray*}
f(K^*)+h(M^*)+g(X^*)+\delta (Z^*)\geq   f(K^{i_j})+h(M^{i_j})+g(X^{i_j})+\delta (Z^{i_j})\\
-\lambda_1\langle K^*-K^{i_j}, K^{i_j}-M^{i_j-1}+b^{i_j-1}_1\rangle
-\lambda_2\langle K^*-K^{i_j}, K^{i_j}-X^{i_j-1}+b^{i_j-1}_2\rangle\\
-\lambda_3\langle K^*-K^{i_j}, K^{i_j}-Z^{i_j-1}+b^{i_j-1}_3\rangle
-\lambda_1\langle M^*-M^{i_j}, M^{i_j}-K^{i_j}-b^{i_j-1}_1\rangle\\
-\lambda_2\langle X^*-X^{i_j}, X^{i_j}-K^{i_j}-b^{i_j-1}_2\rangle-\lambda_3\langle Z^*-Z^{i_j}, Z^{i_j}-K^{i_j}-b^{i_j-1}_3\rangle \\
= f(K^{i_j})+h(M^{i_j})+g(X^{i_j})+\delta (Z^{i_j})-\lambda_1\langle M^{i_j}-K^{i_j}, b^{i_j-1}_1\rangle\\
-\lambda_1\langle K^*-K^{i_j}, K^{i_j}-M^{i_j-1}\rangle
-\lambda_1\langle M^*-M^{i_j}, M^{i_j}-K^{i_j}\rangle\\
-\lambda_2\langle X^{i_j}-K^{i_j}, b^{i_j-1}_2\rangle
-\lambda_2\langle K^*-K^{i_j}, K^{i_j}-X^{i_j-1}\rangle\\
-\lambda_2\langle X^*-X^{i_j}, X^{i_j}-K^{i_j}\rangle
-\lambda_3\langle Z^{i_j}-K^{i_j}, b^{i_j-1}_3\rangle\\
-\lambda_3\langle K^*-K^{i_j}, K^{i_j}-Z^{i_j-1}\rangle
-\lambda_3\langle Z^*-Z^{i_j}, Z^{i_j}-K^{i_j}\rangle,
\end{eqnarray*}
let $j\rightarrow \infty$, then
\begin{equation}\label{conve10-2}
f(K^*)+h(M^*)+g(X^*)+\delta (Z^*)\geq f(\tilde{K})+h(\tilde{M})+g(\tilde{X})+\delta(\tilde{Z}).
\end{equation}
Combine with (\ref{conve10-1}),
$$
f(K^*)+h(M^*)+g(X^*)+\delta (Z^*)= f(\tilde{K})+h(\tilde{M})+g(\tilde{X})+\delta(\tilde{Z}).
$$
Hence, the limit point is minimum of (\ref{conve1}).
\end{proof}

\section*{Acknowledgment}
 This work was supported in part by the HKRGC GRF 12306616, 12200317, 12300218, 12300519, and  CRF C1007-15G. J.Pan acknowledges the support by the European Research Council (ERC starting grant no. 679515). 


%



%
%
%

\bibliographystyle{ieeetr}

\bibliography{ref}


%
%
%
%
%




\end{document}